\def\isarxivversion{1} 
\ifdefined\isarxivversion

\documentclass[12pt]{article}
\usepackage[margin=1.25in]{geometry}
\usepackage{times}
\usepackage{amsmath, amsfonts, bm}
\usepackage{amsthm,graphicx}
\usepackage[dvipsnames]{xcolor}
\usepackage{mathtools}

\usepackage{amssymb}
\usepackage{pifont}

\usepackage{makecell}
\usepackage[utf8]{inputenc} 
\usepackage[T1]{fontenc} 
\usepackage{amsfonts}  
\usepackage{mathrsfs} 
\usepackage{xspace}

\allowdisplaybreaks

\usepackage[unicode=true,
 bookmarks=false,
 breaklinks=false,pdfborder={0 0 1},colorlinks=false]
 {hyperref}
\hypersetup{
 colorlinks,citecolor=blue,filecolor=blue,linkcolor=blue,urlcolor=blue}

\else
\documentclass[anon,12pt]{colt2021}
\fi

\usepackage{algorithm}
\usepackage{array}
\usepackage{makecell}
\usepackage{amssymb}
\usepackage{multirow}
\usepackage{color}
\usepackage[english]{babel}
\usepackage{graphicx}
\usepackage{natbib}
\usepackage{wrapfig}
\usepackage{epstopdf}
\usepackage{url}
\usepackage{graphicx}
\usepackage{color}
\usepackage{epstopdf}
\usepackage{algpseudocode}

\usepackage[T1]{fontenc}
\usepackage{bbm}
\usepackage{comment}
\usepackage{tikz}
\usepackage{tikzit}

\tikzstyle{Default}=[fill=white, draw=black, shape=circle]
\tikzstyle{Rec}=[fill=white, draw=black, shape=rectangle]

\tikzstyle{Unidirectional}=[->]
\tikzstyle{Double Arrow}=[<->]
\tikzstyle{Line}=[-]
\tikzstyle{Dashed arrow}=[->, dashed]
\tikzstyle{Dashed double arrow}=[<->, dashed]
\tikzstyle{dashed line}=[-, dashed]
\tikzstyle{blue line}=[-, color=blue]
\tikzstyle{Dotted Arrow}=[->, dotted]

\usepackage{booktabs,caption}
\usepackage[flushleft]{threeparttable}

\usepackage{tikz}
\usepackage{hyperref}
\hypersetup{colorlinks=true,citecolor=blue,linkcolor=red}

\usetikzlibrary{arrows}

\ifdefined\isarxivversion
\else 

\fi

\definecolor{mygreen}{RGB}{80,180,0}

\definecolor{dacong}{RGB}{10,103,68}

\ifdefined\isarxivversion
\usepackage{amsmath}
\usepackage{amsthm}
\usepackage{subfig}

\usepackage{enumitem}

\newtheorem{theorem}{Theorem}
\newtheorem{lemma}{Lemma}
\newtheorem{definition}{Definition}
\newtheorem{proposition}{Proposition}
\newtheorem{corollary}{Corollary}

\newtheorem{remark}{Remark}
\newtheorem{example}{Example}
\else
\fi

\newtheorem{assumption}{Assumption}

\usepackage{mdframed}
\mdtheorem{problem}{Problem}

\newcommand{\cS}{\mathcal{S}}
\newcommand{\cA}{\mathcal{A}}

\newcommand{\wh}{\widehat}
\newcommand{\wt}{\widetilde}
\newcommand{\ov}{\overline}

\renewcommand{\tilde}{\wt}
\renewcommand{\hat}{\wh}
\renewcommand{\bar}{\ov}

\newcommand{\VC}{\mathcal{V}}
\newcommand{\WC}{\mathcal{W}}
\newcommand{\BW}{B_{w,\alpha}}
\newcommand{\BF}{B_{f,\alpha}}
\newcommand{\BFP}{B_{f',\alpha}}
\newcommand{\BV}{B_{v,\alpha}}
\newcommand{\BE}{B_{e,\alpha}}
\newcommand{\va}{v^*_{\alpha}}
\newcommand{\wa}{w^*_{\alpha}}
\newcommand{\vz}{v^*_{0}}
\newcommand{\ve}{v^*_{\alpha_{\epsilon}}}
\newcommand{\wz}{w^*_{0}}
\newcommand{\we}{w^*_{\alpha_{\epsilon}}}
\newcommand{\pia}{\pi^*_{\alpha}}
\newcommand{\piz}{\pi^*_0}
\newcommand{\pie}{\pi^*_{\alpha_{\epsilon}}}
\newcommand{\La}{L_{\alpha}}
\newcommand{\Lz}{L_0}
\newcommand{\da}{d^*_{\alpha}}
\newcommand{\dz}{d^*_0}
\newcommand{\de}{d^*_{\alpha_{\epsilon}}}

\newcommand{\BFE}{B_{f,\alpha_{\epsilon}}}
\newcommand{\BFZ}{B_{f,0}}
\newcommand{\BWE}{B_{w,\alpha_{\epsilon}}}
\newcommand{\BWZ}{B_{w,0}}
\newcommand{\BEE}{B_{e,\alpha_{\epsilon}}}
\newcommand{\BVE}{B_{v,\alpha_{\epsilon}}}
\newcommand{\BFPZ}{B_{f',0}}
\newcommand{\WCZ}{\mathcal{W}^*_0}
\newcommand{\BS}{B_{w,u}}
\newcommand{\BST}{B_{w,l}}
\newcommand{\PIW}{\Pi_{B_w}}
\newcommand{\piaw}{\pi^*_{\alpha,B_w}}
\newcommand{\pizw}{\pi^*_{0,B_w}}
\newcommand{\vaw}{v^*_{\alpha,B_w}}
\newcommand{\waw}{w^*_{\alpha,B_w}}
\newcommand{\daw}{d^*_{\alpha,B_w}}
\newcommand{\BEZ}{B_{e,0}}

\newcommand{\piew}{\pi^*_{\alpha'_{\epsilon},B_w}}
\newcommand{\vew}{v^*_{\alpha'_{\epsilon},B_w}}
\newcommand{\wew}{w^*_{\alpha'_{\epsilon},B_w}}

\newcommand{\dew}{d^*_{\alpha'_{\epsilon},B_w}}
\newcommand{\vzw}{v^*_{0,B_w}}
\newcommand{\wzw}{w^*_{0,B_w}}
\newcommand{\dzw}{d^*_{0,B_w}}

\newcommand{\Lp}{\tilde{L}_{\alpha}}
\newcommand{\wab}{w^*_{\bar{\alpha}}}
\newcommand{\vab}{v^*_{\bar{\alpha}}}
\newcommand{\vav}{v^*_{\alpha,\mathcal{V}}}
\newcommand{\waww}{w^*_{\alpha,\mathcal{W}}}
\newcommand{\eop}{\epsilon_{{opt}}}
\newcommand{\eag}{\epsilon_{\alpha,{app}}}
\newcommand{\erva}{\epsilon_{\alpha,r,v}}
\newcommand{\erwa}{\epsilon_{\alpha,r,w}}
\newcommand{\eun}{\epsilon_{{un}}}
\newcommand{\eopu}{\epsilon_{{opt}}}
\newcommand{\eagu}{\epsilon_{\alpha_{un},{app}}}
\newcommand{\BFU}{B_{f,\alpha_{un}}}
\newcommand{\BWU}{B_{w,\alpha_{un}}}
\newcommand{\BEU}{B_{e,\alpha_{un}}}
\newcommand{\BVU}{B_{v,\alpha_{un}}}
\newcommand{\eopub}{\epsilon_{{opt}}}
\newcommand{\eagub}{\epsilon_{{app}}}
\newcommand{\lstat}{\epsilon_{stat}}
\newcommand{\lstatt}{\epsilon_{stat,2}}
\newcommand{\wzlp}{w^*}

\newcommand{\FC}{\mathcal{F}}
\newcommand{\TC}{\mathcal{T}}
\newcommand{\SE}{\mathcal{E}}

\newcommand{\mainalg}{\texttt{PRO-RL}\xspace}
\newcommand{\aalg}{\texttt{Inexact-PRO-RL}\xspace}
\newcommand{\algbc}{\texttt{PRO-RL-BC}\xspace}

\definecolor{yxc}{RGB}{255,0,0}
\definecolor{yjc}{RGB}{190,0,255}
\definecolor{whz}{RGB}{1,11,111}
\definecolor{hbh}{RGB}{1,150,20}

\makeatletter
\newcommand*{\RN}[1]{\expandafter\@slowromancap\romannumeral #1@}
\makeatother

\title{}

\ifdefined\usebigfont

\usepackage{times}
\usepackage[fontsize=13pt]{scrextend}
\AtBeginDocument{\newgeometry{letterpaper,left=1.56in,right=1.56in,top=1.71in,bottom=1.77in}}
\else
\usepackage{scrextend}
\fi

\begin{document}

\ifdefined\isarxivversion
\title{Offline Reinforcement Learning with Realizability and Single-policy Concentrability
}

\author{Wenhao Zhan\thanks{Princeton University.}\\
	\and
	\hspace{0.3in}Baihe Huang\thanks{Peking University.}\\
	\and
    \hspace{0.3in}Audrey Huang\thanks{University of Illinois Urbana-Champaign.}\\
	\and
	\hspace{0.3in}Nan Jiang\footnotemark[3]\\
	\and
	\hspace{0.3in}Jason D. Lee\footnotemark[1] \\
	}

\date{February 9, 2022}

\else
\fi

%

\maketitle


\begin{abstract}
Sample-efficiency guarantees for offline reinforcement learning (RL) often rely on strong assumptions on both the function classes (e.g., Bellman-completeness) and the data coverage (e.g., all-policy concentrability). Despite the recent efforts on relaxing these assumptions, existing works are only able to relax one of the two factors, leaving the strong assumption on the other factor intact. As an important open problem, can we achieve sample-efficient offline RL with weak assumptions on \textit{both} factors?

In this paper we answer the question in the positive. We analyze a simple algorithm based on the primal-dual formulation of MDPs, where the dual variables (discounted occupancy) are modeled using a density-ratio function against offline data. With proper regularization, we show that the algorithm enjoys polynomial sample complexity, under \textit{only} realizability 
and single-policy concentrability. 
We also provide alternative analyses based on different assumptions to shed light on the nature of primal-dual algorithms for offline RL. 
\end{abstract}
%

\setcounter{tocdepth}{2}

\section{Introduction} \label{sec:intro}

Offline (or batch) reinforcement learning (RL) learns decision-making strategies using solely historical data, and is a promising framework for applying RL to many real-world applications. Unfortunately, offline RL training is known to be difficult and unstable \citep{fujimoto19off,wang2020statistical,wang2021instabilities}, primarily due to two fundamental challenges.
The first challenge is distribution shift, that the state distributions induced by the candidate
policies may deviate from the offline data distribution, creating difficulties in accurately
assessing the performance of the candidate policies. The second challenge is the sensitivity
to function approximation, that errors can amplify exponentially over the horizon even with
good representations \citep{du2019good,weisz2020exponential,wwk21}. 

These challenges not only manifest themselves as degenerate behaviors of practical algorithms,  
but are also reflected in the strong assumptions needed for providing  sample-efficiency guarantees to classical algorithms. 
(In this paper, by sample-efficiency we mean a sample complexity that is polynomial in the relevant parameters, including the horizon, the capacities of the function classes, and the degree of data coverage.)
As an example, the guarantees of the popular Fitted-Q Iteration \citep{ernst05tree, munos2008finite, pmlr-v120-yang20a, chen2019information} require the following two assumptions:
\begin{itemize}[leftmargin=*, itemsep=0pt]
\item \textbf{(Data) All-policy concentrability: } The offline data distribution provides good coverage (in a technical sense) over the state distributions induced by \textit{all} candidate policies. 
\item \textbf{(Function Approximation) Bellman-completeness: } The value-function class is \textit{closed} under the Bellman optimality operator.\footnote{Approximate policy iteration algorithms usually require a variant of this assumption, that is, the closure under the policy-specific Bellman operator for \textit{every} candidate policy \citep{munos2003error,antos08learning}.} 
\end{itemize}
Both assumptions are very strong and may fail in practice, and algorithms whose guarantees rely on them naturally suffer from performance degradation and instability \citep{fujimoto19off,wang2020statistical,wang2021instabilities}. On one hand, \textit{all-policy concentrability} not only requires a highly exploratory dataset (despite that historical data in real applications often lacks exploration), but also implicitly imposes structural assumptions on the MDP dynamics \citep[Theorem 4]{chen2019information}. On the other hand, \textit{Bellman-completeness} is much stronger than realizability (that the optimal value function is simply contained in the function class), and is \textit{non-monotone} in the function class, that the assumption can be violated more severely when a richer function class is used.

To address these challenges, a significant amount of recent efforts in offline RL have been devoted to relaxing these strong assumptions via novel algorithms and analyses. Unfortunately, these efforts are only able to address either the data or the function-approximation assumption, and no existing works address both simultaneously. For example, \cite{liu2020provably, rajaraman2020towards, rashidinejad2021bridging, jin2020pessimism, xie2021bellman, uehara2021pessimistic} show that pessimism is an effective mechanism for mitigating the negative consequences due to lack of data coverage, and provide guarantees under \textit{single-policy concentrability}, that the data only covers a single good policy (e.g., the optimal policy). However, they require completeness-type assumptions on the value-function classes or \textit{model} realizability.\footnote{When a model class that contains the true MDP model is given, value-function classes that satisfy a version of Bellman-completeness can be automatically induced from the model class \citep{chen2019information}, so model realizability is even stronger than Bellman-completeness. Therefore, in this work we aim at only making a constant number of realizability assumptions of real-valued functions.
} \citet{xie2020batch} only require realizability of the optimal value-function, but their data assumption is even stronger than all-policy concentrability. To this end, we want to ask:
\begin{center}
\textbf{\textit{Is sample-efficiency possible with realizability and single-policy concentrability?}}
\end{center}

In this work, we answer the question in the positive by proposing the first model-free algorithm that only requires relatively weak assumptions on both data coverage and function approximation. 
The algorithm is based on the primal-dual formulation of linear programming (LP) for MDPs \citep{puterman2014markov,wang2017primal}, where we use marginalized importance weight (or density ratio) to model the dual variables which correspond to the discounted occupancy of the learned policy, a practice commonly found in the literature of off-policy evaluation (OPE) \citep[e.g.,][]{liu2018breaking}. Our main result (Corollary~\ref{cor:sample unregularized}) provides polynomial sample-complexity guarantees when the density ratio and (a regularized notion of) the value function of the regularized optimal policy are realizable, and the data distribution covers such an optimal policy. We also provide a number of extensions and alternative analyses to complement the main result and provide deeper understanding of the behavior of primal-dual algorithms in the offline learning setting: (see also  Table~\ref{tab:main results} for a summary of the results) 
\begin{enumerate}[leftmargin=*, itemsep=0pt]
\item Section~\ref{sec:results agnostic} extends the main result to account for approximation and optimization errors, and Section~\ref{sec:results constrained} handles the scenario where the optimal policy is not covered and we need to compete with the best policy supported on data. 
\item Section~\ref{sec:results BC} handles the case where the behavior policy is unknown (which the main algorithm needs) and estimated by behavior cloning.
\item Our main result crucially relies on the use of  regularization. In Section~\ref{sec:results unregularized} we study the unregularized algorithm, and provide performance guarantees under alternative assumptions. 
%
\end{enumerate}

\begin{table}[h!]
	\begin{center}
			\caption{Assumptions required by existing algorithms and our algorithms to learn an $\epsilon$-optimal policy efficiently. Here $\pi$ is a policy and $d^{\pi}$ is the associated discounted state-action occupancy. $\da=d^{\pia}$ where $\pia$ is the $\alpha$-regularized optimial policy (defined in Section~\ref{sec:algorithm}). In particular, $\dz$ is the discounted state-action occupancy of the unregularized optimal policy. $d^D$ is the distribution of the offline dataset. $\FC,\Pi,\WC,\VC$ are the approximation function classes and $\TC$ is the Bellman operator. $Q^{\pi}$ is th action value function of $\pi$ and $Q^*$ is the unregularized optimal action value function. $\va$($\vew$) is the $\alpha$-regularized optimal value function (with respect to the covered policy class), defined in Section~\ref{sec:algorithm} (Section~\ref{sec:results constrained}), and particularly $\vz$ is the unregularized optimal value function. $\wa$($\wew$) is the optimal density ratio $\frac{\da}{d^D}$ (with respect to the covered policy class), as stated in Section~\ref{sec:algorithm} (Section~\ref{sec:results constrained}). Here we compete with the unregularized optimal policy by default and will mark when competing against the regularized optimal policy.}
			\label{tab:main results}
			\begin{tabular}{c|c|c}
					
					Algorithm & Data & Function Class\\
					\hline
					AVI & \multirow{2}{*}{$\Vert\frac{d^{\pi}}{d^D}\Vert_{\infty}\leq B_w,\textcolor{red}{\forall\pi}$} & $\TC f\in\FC,\textcolor{red}{\forall f\in\FC}$ {\scriptsize	\citep{munos2008finite}}\\
					\cline{1-1}\cline{3-3}
					API & &$\TC^{\pi}f\in\FC,\textcolor{red}{\forall f\in\FC,\pi\in\Pi}$ {\scriptsize\citep{antos2008learning}}\\
					\hline
					BVFT & \textcolor{red}{Stronger} than above & $Q^*\in\FC${\scriptsize\citep{pmlr-v139-xie21d}}\\
					\hline
					\multirow{2}{*}{Pessimism} & \multirow{2}{*}{$\Vert\frac{\dz}{d^D}\Vert_{\infty}\leq B_w$} & $\TC^{\pi}f\in\FC,\textcolor{red}{\forall f\in\FC,\pi\in\Pi}$ {\scriptsize\citep{xie2021bellman}}\\
					\cline{3-3}
					& & $\wz\in\WC, Q^{\pi}\in\FC, \textcolor{red}{\forall\pi\in\Pi}${\scriptsize\citep{jiang2020minimax}}\\
					\hline
					\textbf{\mainalg} & \multirow{2}{*}{$\Vert\frac{\da}{d^D}\Vert_{\infty}\leq B_w$} & \multirow{2}{*}{$\wa\in\WC,\va\in\VC$ \scriptsize(\textcolor{blue}{Theorem~\ref{thm:sample regularized}})}\\
					(against $\pia$) & & \\
					\hline
					\multirow{2}{*}{\textbf{\mainalg}} & \multirow{2}{*}{$\Vert\frac{\dz}{d^D}\Vert_{\infty}\leq B_w$} & \multirow{2}{*}{$\wew\in\WC,\vew\in\VC$ \scriptsize(\textcolor{blue}{Corollary~\ref{cor:sample unregularized constrained}})}\\
					&&\\
					\hline
					 & \multirow{2}{*}{$\Vert\frac{\dz}{d^D}\Vert_{\infty}\leq B_w,\frac{\dz(s)}{d^D(s)}\geq B_{w,l},\forall s$} & \multirow{4}{*}{$\wz\in\WC,\vz\in\VC$ \scriptsize(\textcolor{blue}{Corollary~\ref{cor:sample unregularized 0}})}\\
					 \textbf{\mainalg}&&\\ 
					 with $\alpha=0$ &\multirow{2}{*}{$\frac{d^{\pi}(s)}{d^D(s)}\leq B_{w,u},\textcolor{red}{\forall\pi},s$}& \\
					 & & \\
					\hline
				\end{tabular}
		\end{center}
\end{table}
\subsection{Related works}
Section~\ref{sec:intro} has reviewed the analyses of approximate value/policy iteration, and we focus on other related works in this section. 

\paragraph{Lower bounds} When we only assume the realizability of the optimal value-function, a number of recent works have established information-theoretic hardness for offline learning under relatively weak data coverage assumptions \citep{wang2020statistical,amortila2020variant,zanette2021exponential,chen2021infinite}. A very recent result by \citet{foster2021offline} shows a stronger barrier, that even with \textit{all-policy concentrability} and the realizability of the value functions of \textit{all policies}, it is still impossible to obtain polynomial sample complexity in the offline learning setting. These works do not contradict our results, as we also assume the realizability of the density-ratio function, which circumvents the existing lower bound constructions. In particular, as \citet[Section 1.3]{foster2021offline} have commented, their lower bound no longer holds if the realizability of importance weight is assumed, as a realizable weight class would have too large of a capacity in their construction and would explain away the sample-complexity lower bound that scales with $|\cS|$. 


\paragraph{Marginalized importance sampling (MIS)} 
As mentioned above, a key insight that enables us to break the lower bounds against value-function realizability is the use of marginalized importance weights (or density ratio). Modeling such functions is a common practice in MIS, a recently popular approach in the OPE literature \citep{liu2018breaking, uehara2019minimax, kostrikov2019imitation,nachum2020reinforcement,zhang2020dendice}, though most of the works focus exclusively on policy evaluation. 

Among the few works that consider policy optimization, AlgaeDICE \citep{nachum2019algaedice} optimizes the policy using MIS as a subroutine for policy evaluation, and \citet{jiang2020minimax} analyze AlgaeDICE under the realizability of \textit{all} candidate policies' value functions. Similarly, MABO \citep{xie2020q} only needs realizability of the optimal value function, but the weight class needs to realize the density ratio of \textit{all} candidate policies. 
The key difference in our work is the use of the LP formulation of MDPs \citep{puterman2014markov} to directly solve for the optimal policy, without trying to evaluate other policies. This idea has been recently explored by OptiDICE \citep{lee2021optidice}, which is closely related to and has inspired our work. However, \citet{lee2021optidice} focuses on developing an empirical algorithm, and as we will see, multiple design choices in our algorithms deviate from those of OptiDICE and are crucial to obtaining the desired sample-complexity guarantees. 


\section{Preliminaries}
\label{sec:setting}

\paragraph{Markov decision process (MDP).} 
We consider an infinite-horizon discounted MDP $\mathcal{M}=(\mathcal{S},\mathcal{A},P,r,\gamma,\mu_{0})$ \citep{bertsekas2017dynamic}, where $\mathcal{S}$ is the state space, $\mathcal{A}$ is the action space, $\gamma \in [0,1)$ is the discount factor, 
$P: \mathcal{S}\times\mathcal{A}\to\Delta(\mathcal{S})$ is the transition function, $\mu_0\in\Delta(\mathcal{S})$ is the initial state distribution, and $r:\mathcal{S}\times\mathcal{A}\to[0,1]$ is the reward function. Here, we assume $\cS$ and $\cA$ to be finite, but our results will not depend on their cardinalities and 
can be extended to the infinite case naturally. We also assume $\mu_0(s)>0$ for all $s\in\mathcal{S}$; since our analysis and results will not depend on $\min_{s\in\mathcal{S}}\mu_0(s)$, $\mu_0(s)$ for any particular $s$ can be arbitrarily small and therefore this is a trivial assumption for certain technical conveniences.

A policy $\pi:\mathcal{S}\to\Delta(\mathcal{A})$ specifies the action selection probability in state $s$, and the associated discounted state-action occupancy is defined  as
$d^{\pi}(s,a) \coloneqq (1-\gamma) \sum_{t=0}^{\infty} \gamma^t \text{Pr}_\pi( s_t = s, a_t = a),$ 
where the subscript of $\pi$ in $\text{Pr}_{(\cdot)}$ or $\mathbb{E}_{(\cdot)}$ refers to the distribution of trajectories generated as $ s_0\sim \mu_0 $, $a_t\sim\pi(\cdot|s_t)$, $s_{t+1}\sim P(\cdot|s_t,a_t)$ for all $t\geq0$. For brevity, let $d^{\pi}(s)$ denote the discounted state occupancy $\sum_{a\in\mathcal{A}}d^{\pi}(s,a)$. A policy $\pi$ is also associated with a value function $V^{\pi}:\mathcal{S}\to\mathbb{R}$ and an action-value (or Q) function $Q^\pi: \cS\times\cA\to\mathbb{R}$ as follows: $\forall s\in\cS, a\in\cA$,
	%
$\textstyle
V^{\pi}(s):=\mathop{\mathbb{E}}_\pi
\left[\sum_{t=0}^{\infty}\gamma^t r(s_t,a_t) ~\Big\vert~ s_0=s\right],  ~~
Q^\pi(s,a) := \mathop{\mathbb{E}}_\pi\left[\sum_{t=0}^{\infty}\gamma^t r(s_t,a_t) ~\Big\vert~ s_0=s, a_0 = a\right].$ 
	%
	%
%

The goal of RL is to find a policy that maximizes 
the expected discounted return:
\begin{align}
\label{prob:original problem}
\max_{\pi} J(\pi)= (1-\gamma)\mathbb{E}_{\pi}\left[\sum_{t=0}^{\infty}\gamma^tr(s_t,a_t)\right] = \mathop{\mathbb{E}}_{(s,a)\sim d^{\pi}}[r(s,a)].
\end{align}


Alternatively, $J(\pi) =(1-\gamma)V^{\pi}(\mu_0) :=(1-\gamma) \mathbb{E}_{s\sim\mu_0}[V^{\pi}(s)]$. 
Let $\pi^*$ denote the optimal policy of this unregularized problem (\ref{prob:original problem}).



\paragraph{Offline RL.} 
In offline RL, the agent cannot interact with the environment directly and only has access to a pre-collected dataset $\mathcal{D}=\{(s_i,a_i,r_i,s'_i)\}_{i=1}^n$. We further assume each $(s_i,a_i,r_i,s'_i)$ is i.i.d.~sampled from $(s_i,a_i)\sim d^{D},r_i=r(s_i,a_i),s'_i\sim P(\cdot|s_i,a_i)$ as a standard simplification in theory \citep{nachum2019dualdice,nachum2019algaedice,xie2021bellman,pmlr-v139-xie21d}. Besides, we denote the conditional probability $d^D(a|s)$ by $\pi_D(a|s)$ and call $\pi_D$ the behavior policy. However, we do not assume $d^D=d^{\pi_D}$ in most of our results for generality (except for Section~\ref{sec:results constrained}). We also use $d^D(s)$ to represent the marginal distribution of state, i.e., $d^D(s)=\sum_{a\in\mathcal{A}}d^D(s,a)$. In addition, we assume access to a batch of i.i.d. samples $\mathcal{D}_0=\{s_{0,j}\}^{n_0}_{j=1}$ from the initial distribution $\mu_0$.  

\section{Algorithm: \mainalg}
\label{sec:algorithm}
Our algorithm builds on a regularized version of the well-celebrated LP formulation of MDPs \citep{puterman2014markov}. In particular, consider the following problem: 

\begin{samepage}
\begin{problem*}[Regularized LP]\label{prob:regularized_lp}
\vspace{-10pt}
\begin{align}
	&\max_{d \ge 0}\mathbb{E}_{(s,a)\sim d}[r(s,a)]-\alpha\mathbb{E}_{(s,a)\sim d^D}\left[f\left(\frac{d(s,a)}{d^D(s,a)}\right)\right]\label{eq:constrained}\\
	&\text{s.t. }d(s)=(1-\gamma)\mu_0(s)+\gamma\sum_{s',a'}P(s|s',a')d(s',a'), \forall s\in\mathcal{S}\label{eq:bellman flow 1}
\end{align}
where $d\in \mathbb{R}^{|\cS\times\cA|}$, $d(s) = \sum_{a} d(s, a)$, and  $f:\mathbb{R}\to\mathbb{R}$ is a strongly convex and continuously  differentiable function serving as a regularizer. 
\end{problem*}
\end{samepage}

Without the regularization term, this problem is exactly equivalent to the unregularized problem (\ref{prob:original problem}), as (\ref{eq:bellman flow 1}) exactly characterizes the space of possible discounted occupancies $d^\pi$ that can be induced in this MDP and is often known as the Bellman flow equations. Any non-negative $d$ that satisfies such constraints corresponds to $d^\pi$ for some stationary policy $\pi$. Therefore, once we have obtained the optimum $\da$ of the above problem, we can extract the regularized optimal policy $\pia$ via
\begin{equation}
	\label{eq:tilde d tilde w}
	{\pia}(a|s):=
	\begin{cases}
		\frac{{\da}(s,a)}{\sum_a {\da}(s,a)}, & \text{for } \sum_a \da(s,a)>0,\\
		\frac{1}{|\mathcal{A}|}, & \text{else.}
	\end{cases} ~~ \forall s\in\mathcal{S},a\in\mathcal{A}.
\end{equation}

Turning to the regularizer, $D_f(d\Vert d^D):=\mathbb{E}_{(s,a)\sim d^D}\left[f\left(\frac{d(s,a)}{d^D(s,a)}\right)\right]$ is the $f$-divergence between $d^{\pi}$ and $d^D$. This practice, often known as behavioral regularization, encourages the learned policy $\pi$ to induce an occupancy $d = d^\pi$ that stays within the data distribution $d^D$, and we will motivate it further using a counterexample against the unregularized algorithm \& analysis at the end of this section. 

To convert the regularized problem \eqref{eq:constrained}\eqref{eq:bellman flow 1} into a learning algorithm compatible with function approximation, 
we first introduce the Lagrangian multiplier ${v}\in\mathbb{R}^{|\mathcal{S}|}$ to (\ref{eq:constrained})(\ref{eq:bellman flow 1}), and obtain the following maximin problem:
\begin{align}
\label{prob:maximin}
\max_{d\geq0}\min_{{v}}~&\mathbb{E}_{(s,a)\sim d}[r(s,a)]-\alpha\mathbb{E}_{(s,a)\sim d^D}\left[f\left(\frac{d(s,a)}{d^D(s,a)}\right)\right]\notag\\
&+\sum_{s\in\mathcal{S}}{v}(s)\left((1-\gamma)\mu_0(s)+\gamma\sum_{s',a'}P(s|s',a')d(s',a')-d(s)\right).
\end{align}
Then, by variable substitution $w(s,a)=\frac{d(s,a)}{d^D(s,a)}$ and replacing summations with the corresponding expectations, we obtain the following problem
\begin{equation}
\label{prob:maximin2}
\max_{w\geq0}\min_{{v}}{\La}({v},w):=(1-\gamma)\mathbb{E}_{s\sim \mu_0}[{v}(s)]-\alpha\mathbb{E}_{(s,a)\sim d^D}[f(w(s,a))]+\mathbb{E}_{(s,a)\sim d^D}[w(s,a)e_{{v}}(s,a)],
\end{equation}
where $e_{{v}}(s,a)=r(s,a)+\gamma\sum_{s'}P(s'|s,a){v}(s')-{v}(s)$. The optimum of (\ref{prob:maximin2}), denoted by $({\va},{\wa})$, will be of vital importance later, as our main result relies on the realizability of these two functions $\va$ and $\wa$.  
When $\alpha= 0$, $v_{0}^*$ is the familiar optimal state-value function $V^{\pi^*}$, and $d_0^* := w_0^* \cdot d^D$ is the discounted occupancy of an optimal policy. Note that optimal policies in MDPs are generally not unique and thus $w^*_0,d^*_0$ are not unique either. We denote the optimal set of $w^*_0$ and $d^*_0$ by $\WC^*_0$ and $D^*_0$, respectively.

Finally, our algorithm simply uses function classes  ${\VC}\subseteq\mathbb{R}^{|\mathcal{S}|}$ and ${\WC}\subseteq\mathbb{R}^{|\mathcal{S}|\times|\mathcal{A}|}_{+}$ to approximate ${v}$ and $w$, respectively, and optimizes the empirical version of $\La(v, w)$ over ${\WC} \times \VC$. Concretely, we solve for
\begin{align}
\label{prob:empirical}
\textbf{\mainalg:} \qquad (\hat{{w}},\hat{v})=\arg\max_{w\in {\WC}}\arg\min_{{v}\in {\VC}}\hat{L}_{\alpha}({v},w) ,
\end{align}
where $\hat{L}_{\alpha}({v},w):=$
\begin{align}\label{hat L}
(1-\gamma)\frac{1}{n_0}\sum_{j=1}^{n_0}[{v}(s_{0,j})]+\frac{1}{n}\sum_{i=1}^n[-\alpha f(w(s_i,a_i))]
+\frac{1}{n}\sum_{i=1}^n[w(s_i,a_i)e_{{v}}(s_i,a_i,r_i,s'_i)],
\end{align}
and $e_{{v}}(s,a,r,s')=r+\gamma {v}(s')-{v}(s)$. The final policy we obtain is
\begin{align}
	\label{eq:hat pi}
	\hat\pi(a|s)=
	\begin{cases}
		\frac{\hat{w}(s,a)\pi_D(a|s)}{\sum_{a'}\hat{w}(s,a')\pi_D(a'|s)}, & \text{for } \sum_{a'}\hat{w}(s,a')\pi_D(a'|s)>0,\\
		\frac{1}{|\mathcal{A}|}, & \text{else,}
	\end{cases}
\end{align}
We call this algorithm \textbf{P}rimal-dual \textbf{R}egularized \textbf{O}ffline \textbf{R}einforcement \textbf{L}earning (\mainalg). For now we assume the behavior policy $\pi_D$ is known; Section~\ref{sec:results BC} extends the main results to the unknown $\pi_D$ setting via behavior cloning.

\paragraph{Why behavioral regularization?} While  behavioral regularization (the $f$ term) is frequently used in MIS (especially in DICE algorithms \citep{nachum2019algaedice,lee2021optidice}), its theoretical role has been unclear and finite-sample guarantees can often be obtained without it \citep{jiang2020minimax}. For us, however, the use of regularization is crucial in proving our main result (Corollary~\ref{cor:sample unregularized}). Below we construct a counterexample against the unregularized algorithm under the natural ``unregularized'' assumptions. 
\begin{figure}[t]
\centering
\begin{tikzpicture}
	\begin{pgfonlayer}{nodelayer}
		\node [style=Default] (0) at (-8.5, 21.5) {A};
		\node [style=Default] (1) at (-9.25, 20.25) {B};
		\node [style=Default] (3) at (-7.75, 20.25) {C};
		\node [style=none] (6) at (-9.25, 18.5) {+1};
		\node [style=none] (7) at (-8.25, 18.75) {+0};
		\node [style=none] (8) at (-7.25, 18.75) {+1};
		\node [style=none] (9) at (-8.25, 18.25) {+1};
		\node [style=none] (10) at (-7.25, 18.25) {+0};
		\node [style=none] (11) at (-8.75, 20.75) {};
		\node [style=none] (12) at (-8.75, 19) {};
		\node [style=none] (13) at (-6.75, 19) {};
		\node [style=none] (14) at (-6.75, 20.75) {};
		\node [style=none] (15) at (-9.25, 19.25) {};
		\node [style=none] (16) at (-8.25, 19.25) {};
		\node [style=none] (17) at (-7.25, 19.25) {};
		\node [style=none] (19) at (-9, 21) {L};
		\node [style=none] (20) at (-8, 21) {R};
		\node [style=none] (21) at (-8.25, 19.75) {L};
		\node [style=none] (22) at (-7.25, 19.75) {R};
	\end{pgfonlayer}
	\begin{pgfonlayer}{edgelayer}
		\draw [style=Unidirectional] (0) to (1);
		\draw [style=Unidirectional] (0) to (3);
		\draw [style=dashed line] (12.center) to (13.center);
		\draw [style=dashed line] (13.center) to (14.center);
		\draw [style=dashed line, in=360, out=180] (14.center) to (11.center);
		\draw [style=dashed line] (11.center) to (12.center);
		\draw [style=Unidirectional] (1) to (15.center);
		\draw [style=Unidirectional] (3) to (16.center);
		\draw [style=Unidirectional] (3) to (17.center);
	\end{pgfonlayer}
\end{tikzpicture}
\caption{Construction against the unregularized algorithm under $w_0^* \in \WC$ and $v_0^* \in \VC$. The construction is given as a 2-stage finite-horizon MDP, and adaptation to the discounted setting is trivial. State A is the initial state with no intermediate rewards. 
The offline data does not cover state C. The nature can choose between 2 MDPs that differ in the rewards for state C, and only one of the two actions has a $+1$ reward.  \label{fig:counter}}
\end{figure}
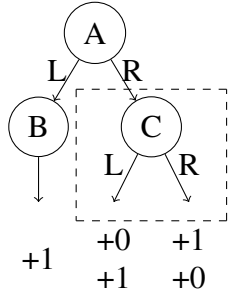

\begin{example}
Figure~\ref{fig:counter} shows a counterexample where the unregularized algorithm fails even with infinite data and the natural assumptions, that (1) there exists a $w_0^*\in\WC^*_0$ such that $w_0^*\in \WC$, (2) $v_0^* \in \VC$, and (3) data covers the optimal policy induced by $w_0^*$. In state A, both actions are equally optimal. However, since data does not cover the actions of state C, the learner should not take \texttt{R} in state A as it can end up choosing a highly suboptimal action in state C with constant probability if nature randomizes over the 2 possible MDP instances. 

We now show that the unregularized algorithm (\eqref{prob:empirical} with $\alpha=0$) can choose \texttt{R} in state A, even with infinite data and ``nice'' $d^D$, $\VC$, $\WC$. In particular, the two possible MDPs share the same optimal value function $v_0^*(A) = v_0^*(B) = v_0^*(C) = 1$, which is the only function in $\VC$ so we always have $v_0^* \in \VC$. $d^D$ covers state-action pairs (A, \texttt{L}), (A, \texttt{R}), B. $\WC$ also contains 2 functions: $w_1$ is such that $w_1 \cdot d^D$ is uniform over (A, \texttt{L}), B, which is the occupancy of the optimal policy $\pi^*(A) = L$. $w_2$ is such that $w_2 \cdot d^D$ is uniform over (A, \texttt{R}), B, which induces a policy that chooses \texttt{R} in state A. However, the unregularized algorithm cannot distinguish between $w_1$ and $w_2$ even with infinite data (i.e., with objective $L_0(v, w)$). This is because $w_1$ and $w_2$ only differs in the action choice in state A, but $v_0^*(B) = v_0^*(C) = 1$ so the unregularized objective is the same for $w_1$ and $w_2$. 
\end{example}

\section{Main results}
\label{sec:results}
In this section we present the main sample-complexity guarantees of our algorithm 
under only realizability assumptions for $\VC$ and $\WC$ and single-policy concentrability of data. We will start with the analyses that assume perfect optimization and that the behavior policy $\pi_D$ is known (Section~\ref{sec:results regularized}), allowing us to present the result in a clean manner. 
We then extend our analyses in several directions: Section~\ref{sec:results agnostic} handles  approximation and optimization errors; Section~\ref{sec:results constrained} removes the concentrability assumption altogether and allows us to compete with the best covered policy;  Section~\ref{sec:results BC} uses behavior cloning to handle an unknown behavior policy. 

\subsection{Sample-efficiency with only realizability and weak concentrability}
\label{sec:results regularized}
We introduce the needed assumptions before stating the sample-efficiency guarantees to our algorithm. The first assumption is about data coverage, that it covers the occupancy induced by a (possibly regularized) optimal policy.
 
\begin{assumption}[$\pia$-concentrability]
	\label{ass:concentrability}	
	\begin{equation}
		\frac{{\da}(s,a)}{d^D(s,a)}\leq\BW, \forall s\in\mathcal{S},a\in\mathcal{A}.
	\end{equation}
\end{assumption}
Two remarks are in order:
\begin{enumerate}[leftmargin=*,itemsep=0pt]
\item Assumption~\ref{ass:concentrability}	 is parameterized by $\alpha$, and we will bind it to specific values when we state the guarantees. 
\item This assumption is necessary if we want to compete with the optimal policy of the MDP, $\pi^*$, and is already much weaker than all-policy concentrability \citep{munos2008finite, farahmand2010error, chen2019information}. That said, ideally we should not even need such an assumption, as long as we are willing to compete with the best policy covered by data instead of the truly optimal policy \citep{liu2020provably, xie2021bellman}. We will actually show how to achieve this in Section~\ref{sec:results constrained}. 
\end{enumerate}

We then introduce the realizability assumptions on our function approximators $\VC$ and $\WC$, which are very straightforward. For now we assume exact realizability, and Section~\ref{sec:results agnostic} handles misspecification errors.

\begin{assumption}[Realizability of ${\VC}$]
	\label{ass:V realize}
	Suppose ${\va}\in {\VC}$. 
\end{assumption}  
\begin{assumption}[Realizability of ${\WC}$]
	\label{ass:W realize}
	Suppose ${\wa}\in {\WC}$.
\end{assumption}

The above 3 assumptions are the major assumptions we need. (The rest are standard technical assumptions on boundedness.) Comparing them to existing results, we emphasize that all existing analyses require ``$\forall''$ quantifiers  in the assumptions either about the data (e.g., all-policy concentrability) or about the function classes (e.g., Bellman-completeness). See Table~\ref{tab:main results} for a comparison to various approaches considered in the literature.


Having stated the major assumptions, we now turn to the routine ones on function boundedness. 
\begin{assumption}[Boundedness of ${\WC}$]
\label{ass:W bound}
Suppose $0\leq w(s,a)\leq {\BW}$ for any $s\in\mathcal{S}, a\in\mathcal{A}, w\in {\WC}$.
\end{assumption}
Here we reuse $\BW$ from Assumption~\ref{ass:concentrability}. Since $\da/d^D = \wa \in \WC$ by Assumption~\ref{ass:W realize}, in general the magnitude of $\WC$ should be larger than that of $\da/d^D$, and we use the same upper bound to eliminate unnecessary notations and improve readability.

The next assumption characterizes the regularizer $f$. These are not really assumptions as we can make concrete choices of $f$ that satisfy them (e.g., a simple quadratic function; see Remark~\ref{rem:sample regularized}), but for now we leave them as assumptions to keep the analysis general. 
\begin{assumption}[Properties of $f$]
\label{ass:f prop}	
Suppose $f$ satisfies the following properties:
\begin{itemize}
\item \textbf{Strong Convexity}: $f$ is $M_f$-strongly-convex.
\item \textbf{Boundedness}:
\begin{align}
	&|f'(x)|\leq {\BFP},\forall\quad 0\leq x\leq {\BW},\\
	&|f(x)|\leq {\BF}, \forall\quad 0\leq x\leq {\BW}.
\end{align}
\item \textbf{Non-negativity}: $f(x)\geq0$ for any $x\in\mathbb{R}$.
\end{itemize}
\end{assumption}
\begin{remark}
The non-negativity is trivial since $f$ is strongly convex and we can always add a constant term to ensure non-negativity holds. Besides, we can get rid of non-negativity with the results in Section~\ref{sec:results constrained}.
\end{remark}

Assumption~\ref{ass:f prop} allows us to bound $\|\va\|_\infty \le \frac{\alpha {\BFP}+1}{1-\gamma}$ (see Lemma~\ref{lem:bound tilde nu} in  Section~\ref{sec:analysis}); in the same spirit as Assumption~\ref{ass:W bound}, we assume:
\begin{assumption}[Boundedness of ${\VC}$]
\label{ass:V bound}
Suppose $\Vert v\Vert_{\infty}\leq {\BV}:=\frac{\alpha {\BFP}+1}{1-\gamma}$ for any $v\in {\VC}$.
\end{assumption}

With the above assumptions, we have Theorem~\ref{thm:sample regularized} to show that \mainalg can learn the optimal density ratio and policy for the regularized problem \eqref{eq:constrained}\eqref{eq:bellman flow 1} with polynomial samples, whose proof is deferred to Section~\ref{sec:analysis}. To simplify writing, we introduce the following notation for the statistical error term that arises purely from concentration inequalities:
\begin{definition}
\begin{align}
&\SE_{n,n_0,\alpha}(B_w,B_f,B_v,B_e)=(1-\gamma)B_v\cdot\left(\frac{2\log\frac{4|{\VC}|}{\delta}}{n_0}\right)^{\frac{1}{2}}+\left(\alpha B_f+{B_w}B_e\right)\cdot\left(\frac{2\log\frac{4|{\VC}||{\WC}|}{\delta}}{n}\right)^{\frac{1}{2}}.
\end{align}
\end{definition}
$\SE$ characterizes the statistical error $\hat{L}_{\alpha}(v,w)-L_{\alpha}(v,w)$ based on concentration inequalities, and the two terms in its definition correspond to 
using $\mathcal{D}_0$ to approximate $(1-\gamma)\mathbb{E}_{s\sim \mu_0}[{v}(s)]$ and $\mathcal{D}$ for $-\alpha\mathbb{E}_{(s,a)\sim d^D}[f(w(s,a))]+\mathbb{E}_{(s,a)\sim d^D}[w(s,a)e_{{v}}(s,a)]$, respectively. 
Using this shorthand, we state our first guarantee, that the learned $\hat w$ and the extracted policy $\hat\pi$ will be close to the solution of the regularized problem \eqref{eq:constrained}\eqref{eq:bellman flow 1}, $\wa$ and $\pia$, respectively. 

\begin{theorem}[Sample complexity of learning $\pi_\alpha^*$]
	\label{thm:sample regularized}
	Fix $\alpha>0$. Suppose Assumptions~\ref{ass:concentrability},\ref{ass:V realize},\ref{ass:W realize},\ref{ass:W bound},\ref{ass:f prop},\ref{ass:V bound} hold for the said $\alpha$. Then with at least probability $1-\delta$, the output of \mainalg satisfies:
	\begin{align}
		&J(\pia)-J(\hat{\pi})\leq\frac{1}{1-\gamma}\mathbb{E}_{s\sim {\da}}[\Vert{\pia}(\cdot|s)-\hat{\pi}(\cdot|s)\Vert_1]\notag\\
		&\leq\frac{2}{1-\gamma}\Vert \hat{w}-{\wa}\Vert_{2,d^D}\leq\frac{4}{1-\gamma}\sqrt{\frac{\SE_{n,n_0,\alpha}(\BW,\BF,\BV,\BE)}{\alpha M_f}},\label{eq:thm1 1}
	\end{align}
	where 
	${\BE}:=(1+\gamma){\BV}+1$.
\end{theorem}

\begin{remark}[Sample complexity for quadratic regularization]
\label{rem:sample regularized}
Theorem~\ref{thm:sample regularized} shows that \mainalg can obtain a near-optimal policy for regularized problem \eqref{eq:constrained}\eqref{eq:bellman flow 1} with sample complexity $O(n_0+n_1)=\tilde{O}\left(\frac{(\alpha {\BF}+{\BW}{\BE})^2}{(1-\gamma)^4(\alpha M_f)^2\epsilon^4}\right)$. However, there might be implicit dependence on $1-\gamma,\alpha M_f,{\BW}$ in the constants ${\BE}$. To reveal these terms, we consider a simple choice of $f(x)=\frac{M_f}{2}x^2$. Then we have ${\BE}=O(\frac{\alpha M_f({\BW})^2+{\BW}}{1-\gamma}),{\BF}=O(\alpha M_f({\BW})^2)$, leading to a sample complexity $\tilde{O}\Big(\frac{({\BW})^2}{(1-\gamma)^6(\alpha M_f)^2\epsilon^4}+\frac{({\BW})^4}{(1-\gamma)^6\epsilon^4}\Big)$.
\end{remark}

Moreover, \mainalg can even learn a near-optimal policy for the unregularized problem (\ref{prob:original problem}) efficiently by controlling the magnitude of $\alpha$ in \mainalg.
Corollary~\ref{cor:sample unregularized} characterizes the sample complexity of \mainalg for the unregularized problem \eqref{prob:original problem} without any approximation/optimization error:
\begin{corollary}[Sample complexity of competing with $\pi_0^*$]
	\label{cor:sample unregularized}
	Fix any $\epsilon>0$. Suppose there exists $d^*_0\in D^*_0$ that satisfies Assumption~\ref{ass:concentrability} with $\alpha=0$. Besides, assume that Assumptions~\ref{ass:concentrability},\ref{ass:V realize},\ref{ass:W realize},\ref{ass:W bound},\ref{ass:f prop},\ref{ass:V bound} hold for $\alpha=\alpha_{\epsilon}:=\frac{\epsilon}{2{\BFZ}}$. Then if 
	\begin{align}
	&n\geq\frac{C_1\left(\epsilon {\BFE}+2{\BWE}{\BEE}{\BFZ}\right)^2}{\epsilon^6M_f^2(1-\gamma)^4}\cdot\log\frac{4|\VC||\WC|}{\delta}, \\
	&n_0\geq\frac{C_1\left(2{\BVE}{\BFZ}\right)^2}{\epsilon^6M_f^2(1-\gamma)^2}\cdot\log\frac{4|\VC|}{\delta},
	\end{align}
	the output of \mainalg with input $\alpha=\alpha_{\epsilon}$ satisfies
	\begin{equation}
		J(\piz)-J(\hat{\pi})\leq\epsilon,
	\end{equation}
	with at least probability $1-\delta$, where $C_1$ is some universal positive constants and $\piz$ is the optimal policy inducing $d^*_0$.
\end{corollary}
\begin{proof}[Proof sketch]
	The key idea is to let $\alpha$ be sufficiently small so that $J({\pia})$ and $J({\piz})$ is close. Then we can simply apply Theorem~\ref{thm:sample regularized} and bound $J(\hat{\pi})-J({\pia})$. See Appendix~\ref{proof:cor sample unregularized} for details.
\end{proof}
\begin{remark}[Quadratic regularization]\label{rem:sample_unreg}
	Similarly as Remark~\ref{rem:sample regularized}, the sample complexity of competing with $\pie$ under quadratic $f$  is $\tilde{O}\left(\frac{({\BWZ})^4({\BWE})^2}{\epsilon^6(1-\gamma)^{6}}\right)$.
\end{remark}
\mainalg is originally designed for the regularized problem. Therefore, when applying it to the unregularized problem the sample complexity degrades from $\tilde{O}\left(\frac{1}{\epsilon^4}\right)$ to $\tilde{O}\left(\frac{1}{\epsilon^6}\right)$. However, the sample complexity remains polynomial in all relevant quantities.
Compared to Theorem~\ref{thm:sample regularized}, Corollary~\ref{cor:sample unregularized} requires concentrability for policy $\piz$ in addition to $\pie$, so technically we require ``two-policy'' instead of single-policy concentrability for now. 
While this is still much weaker than all-policy concentrability \citep{chen2019information}, we show in Section~\ref{sec:results constrained} how to compete with $\piz$ with only single-policy concentrability. 

\begin{remark}
	When $\epsilon$ shrinks, the realizability assumptions for Corollary~\ref{cor:sample unregularized} also need to hold for regularized solutions with smaller $\alpha$. That said, in the following discussion (Proposition~\ref{prop:LP stability}), we will show that when $\epsilon$ is subsequently small, the realizability assumptions will turn to be with respect to the unregularized solutions.
\end{remark}

\paragraph{Comparison with existing algorithms.}
Theorem~\ref{thm:sample regularized} and Corollary~\ref{cor:sample unregularized} display an exciting result that \mainalg obtains a near optimal policy for regularized problem \eqref{eq:constrained}\eqref{eq:bellman flow 1} and unregularized problem \eqref{prob:original problem} using polynomial samples with only realizability and weak data-coverage assumptions. The literature has demonstrated hardness of learning offline RL problems and existing algorithms either rely on the completeness assumptions \citep{xie2020q,xie2021bellman,du2021bilinear} or extremely strong data assumption \citep{pmlr-v139-xie21d}. Our results show for the first time that offline RL problems can be solved using a polynomial number of samples without these assumptions.


\paragraph{High accuracy regime ($\epsilon\to0$).}
Corollary~\ref{cor:sample unregularized} requires weak concentration and realizability with respect to the optimizers of the regularized problem \eqref{eq:constrained}\eqref{eq:bellman flow 1}. A natural idea is to consider whether the concentration and realizability instead can be with respect to the optimizer of the unregularized problem $(\vz,\wz)$. Inspired by the stability of linear programming \citep{mangasarian1979nonlinear}, we identify the high accuracy regime ($\epsilon\to0$) where concentrability and realizability with respect to $\wz$ can guarantee \mainalg to output an $\epsilon$-optimal policy as shown in the following proposition:
\begin{proposition}
	\label{prop:LP stability}
	There exists $\bar{\alpha}>0$ and $w^*\in \WCZ$ such that when $\alpha\in[0,\bar{\alpha}]$ we have
	\begin{equation}
		\wa=w^*,\Vert\va-\vz\Vert_{2,d^D}\leq C\alpha,
	\end{equation}
	where $C=\frac{\BFPZ+\frac{2}{\bar{\alpha}}}{1-\gamma}$.
\end{proposition}
\begin{proof}
	$w^*$ is indeed the solution of $\arg\max_{w\in {\WCZ}}-\alpha\mathbb{E}_{(s,a)\sim d^D}[f(w(s,a))]$ and it can be shown that $w^*$ satisfies the KKT condition of the regularized problem \eqref{eq:constrained}\eqref{eq:bellman flow 1} for sufficiently small $\alpha$. See Appendix~\ref{proof:prop LP stability} for details.
\end{proof}

Here $\bar{\alpha}$ is a value only depends on the underlying MDP and not on $\epsilon$. Proposition~\ref{prop:LP stability} essentially indicates that when $\epsilon\to0$, $\we$ is exactly the unregularized optimum $w^*$, and $\ve$ is $O(\epsilon)$ away form $\vz$. Combining with Corollary~\ref{cor:sample unregularized}, we know that $\epsilon$-optimal policy can be learned by \mainalg if concentrability holds for $\piz$ and $\WC$ contains $w^*$.

\subsection{Robustness to approximation and optimization errors}
\label{sec:results agnostic}
In this section we consider the setting where ${\VC}\times {\WC}$ may not contain $({\va},{\wa})$ and measure the approximation errors as follows:
\begin{align}
	&\erva=\min_{{v}\in {\VC}}\Vert {v}-{\va}\Vert_{1,\mu_0}+\Vert {v}-{\va}\Vert_{1,d^D}+\Vert {v}-{\va}\Vert_{1,d^{D'}},\\
	&\erwa=\min_{w\in {\WC}}\Vert w-{\wa}\Vert_{1,d^D},
\end{align}
where $d^{D'}(s)=\sum_{s',a'}d^D(s',a')P(s|s',a'),\forall s\in\mathcal{S}$. Notice that our definitions of approximation errors are all in $\ell_1$ norm and  weaker than $\ell_\infty$ norm error. 

%

Besides, to make our algorithm work in practice, we also assume $(\hat{{v}},\hat{w})$ is an approximate solution of $\hat{L}_{\alpha}({v},w)$: 
\begin{align}
	&\hat{L}_{\alpha}(\hat{{v}},\hat{w})-\min_{{v}\in {\VC}}\hat{L}_{\alpha}({v},\hat{w})\leq\epsilon_{o,{v}},\label{AU-requirement-1}\\
	&\max_{w\in {\WC}}\min_{{v}\in {\VC}}\hat{L}_{\alpha}({v},w)-\min_{{v}\in {\VC}}\hat{L}_{\alpha}({v},\hat{w})\leq\epsilon_{o,w}.\label{AU-requirement-2}
\end{align}
Equation \eqref{AU-requirement-1} says that $\hat L_\alpha (\hat v, \hat w) \approx \min_v\hat L_\alpha (v, \hat w)$. Equation \eqref{AU-requirement-2} says that $\min_v \hat L_\alpha(v, \hat w)\approx \max_{w\in {\WC}}\min_{{v}\in {\VC}}\hat{L}_{\alpha}({v},w)$. Combining these gives $\hat L_\alpha( \hat v, \hat w) \approx \max_{w\in\WC}\min_{v\in\VC}\hat L_\alpha (v,w)$, so $(\hat v, \hat w)$ is approximately a max-min point.

In this case we call the algorithm \aalg. Theorem~\ref{thm:sample regularized approximate} shows that \aalg is also capable of learning a near-optimal policy with polynomial sample size:

\begin{theorem}[Error-robust version of Theorem~\ref{thm:sample regularized}]
	\label{thm:sample regularized approximate}
	Assume $\alpha>0$. Suppose Assumption~\ref{ass:concentrability},\ref{ass:W bound},\ref{ass:f prop},\ref{ass:V bound} hold. Then with at least probability $1-\delta$, the output of \aalg satisfies:
	\begin{align}
		&J(\pia)-J(\hat{\pi})\leq\frac{1}{1-\gamma}\mathbb{E}_{s\sim {\da}}[\Vert{\pia}(\cdot|s)-\hat{\pi}(\cdot|s)\Vert_1]\leq \frac{2}{1-\gamma}\Vert \hat{w}-{\wa}\Vert_{2,d^D}\notag\\
		&\leq\frac{4}{1-\gamma}\sqrt{\frac{\SE_{n,n_0,\alpha}(\BW,\BF,\BV,\BE)}{\alpha M_f}}+\frac{2}{1-\gamma}\sqrt{\frac{2(\epsilon_{opt}+\eag)}{\alpha M_f}},
	\end{align}
	where ${\BE}$ is defined as Theorem~\ref{thm:sample regularized}, $\eop=\epsilon_{o,{v}}+\epsilon_{o,w}$ and $\eag=\left({\BW}+1\right)\erva+({\BE}+\alpha {\BFP})\erwa$.
\end{theorem}

\begin{proof}[Proof sketch]
The proof follows similar steps in the proof of Theorem~\ref{thm:sample regularized}. See Appendix~\ref{proof:thm sample regularized approximate} for details.
\end{proof}

\begin{remark}[Optimization]
When ${\WC}$ and ${\VC}$ are convex sets,\footnote{In this case they are infinite classes, and we can simply replace  the concentration bound in Lemma~\ref{lem:hat L conc} with a standard covering argument} a line of algorithms \citep{nemirovski2004prox,nesterov2007dual,lin2020near} are shown to attain $\tilde{\epsilon}$-saddle point with the gradient complexity of $\tilde{O}(\frac{1}{\tilde{\epsilon}})$. Notice that an approximate saddle point will satisfy our requirements (\ref{AU-requirement-1})(\ref{AU-requirement-2}) automatically, therefore we can choose these algorithms to solve $(\hat{v},\hat{w})$. In more general cases, ${\WC}$ and ${\VC}$ might be parameterized by $\theta$ and $\phi$. As long as the corresponding maximin problem~(\ref{prob:empirical}) is still concave-convex (e.g., ${\WC}$ and ${\VC}$ are linear function classes), these algorithms can still work efficiently.
\end{remark}

Similar to Corollary~\ref{cor:sample unregularized}, we can extend Theorem~\ref{thm:sample regularized approximate} to compete with $\pi_0^*$. Suppose we select $\alpha=\alpha_{un}>0$ in \aalg and let $\eun=\alpha_{un}\BFZ+\frac{2}{1-\gamma}\sqrt{\frac{2(\eopu+\eagu)}{\alpha_{un} M_f}}$. Then we have the following corollary:

\begin{corollary}[Error-robust version of Corollary~\ref{cor:sample unregularized}] 
	\label{cor:sample unregularized approximate}
	Fix $\alpha_{un} > 0$. Suppose there exists $d^*_0\in D^*_0$ such that Assumption~\ref{ass:concentrability} holds. Besides, assume that Assumptions~\ref{ass:concentrability},\ref{ass:W bound},\ref{ass:f prop},\ref{ass:V bound} hold for $\alpha=\alpha_{un}$. Then 
	the output of \aalg with input $\alpha=\alpha_{un}$ satisfies
	\begin{align}
		&J(\piz)-J(\hat{\pi})\leq\frac{4}{1-\gamma}\sqrt{\frac{\SE_{n,n_0,\alpha_{un}}(\BWU,\BFU,\BVU,\BEU)}{\alpha_{un}M_f}}+\eun,
	\end{align}
	with at least probability $1-\delta$.
\end{corollary}

\begin{proof}[Proof sketch]
	The proof largely follows that of Corollary~\ref{cor:sample unregularized} and thus is omitted here.
\end{proof}

\paragraph{The selection of $\alpha_{un}$}
The best $\alpha_{un}$ we can expect (i.e., with the lowest error floor) is
\begin{equation}
\alpha_{un}:=\arg\min_{\alpha>0}\left(\alpha\BFZ+\frac{2}{1-\gamma}\sqrt{\frac{2(\eop+\eag)}{\alpha M_f}}\right).
\end{equation}
However, this requires knowledge of $\eag$, which is often unknown in practice. One alternative method is to suppose $\eag$ upper bounded by $\eagub$ for some $\alpha\in\ I_{\alpha}$, then $\alpha_{un}$ can be chosen as
\begin{equation}
\alpha_{un}:=\arg\min_{\alpha\in I_{\alpha}}\left(\alpha\BFZ+\frac{2}{1-\gamma}\sqrt{\frac{2(\eop+\eagub)}{\alpha M_f}}\right).
\end{equation}
Notice that $\BFZ$ is known and $\eop$ can be controlled by adjusting the parameters of the optimization algorithm, therefore the above $\alpha_{un}$ can be calculated easily. 

\paragraph{Higher error floor}
In the ideal case of no approximation/optimization errors, Corollary~\ref{cor:sample unregularized} (which competes with $\pi_0^*$) has a worse sample complexity than Theorem~\ref{thm:sample regularized} (which only competes with $\pia$). However, with the presence of approximation and optimization errors, the sample complexities become the same in Theorem~\ref{thm:sample regularized approximate} and Corollary~\ref{cor:sample unregularized approximate}, but the latter has a higher error floor. To see this, we can suppose $\eag$ are uniformly upper bounded by $\eagub$, then $\alpha_{un}=O((\eopub+\eagub)^{\frac{1}{3}})$ by the AM-GM inequality and $\eun=O((\eopub+\eagub)^{\frac{1}{3}})$, which is larger than $O((\eopub+\eagub)^{\frac{1}{2}})$ as in Theorem~\ref{thm:sample regularized approximate}.

\subsection{Handling an arbitrary data distribution}
\label{sec:results constrained}
In the previous sections, our goal is to compete with policy ${\pia}$ and we require the data to provide sufficient coverage over such a policy. Despite being weaker than all-policy concentrability, this assumption can be still violated in practice, since we have no control over the distribution of the offline data. In fact,  recent works such as \citet{xie2021bellman} are able to compete with the best policy covered by data (under strong function-approximation assumptions such as Bellman-completeness), thus provide guarantees to \textit{arbitrary} data distributions: when the data does not cover any good policies, the guarantee is vacuous; however, as long as a good policy is covered, the guarantee will be competitive to such a policy. 

In this section we show that we can achieve similar guarantees for \mainalg with a twisted analysis. 
First let us define the notion of covered policies.
\begin{definition}
Let ${\PIW}$ denote the $B_w$-covered policy class of $d^D$ for $B_w>1$, defined as:
\begin{align}
{\PIW}\coloneqq\{\pi:\frac{d^{\pi}(s,a)}{d^D(s,a)}\leq B_w,\forall s\in\mathcal{S},a\in\mathcal{A}\}.
\end{align}
\end{definition}
Here, $B_w$ is a hyperparameter chosen by the practitioner, and our goal in this section is to compete with policies in $\PIW$. 
The key idea is to extend the regularized LP \eqref{eq:constrained} by introducing an additional upper-bound constraint on $d$, that $d(s,a) \le B_w d^D$, so that we only search for a good policy within $\PIW$. The policy we will compete with  $\piaw$ and the corresponding value and density-ratio functions $\vaw$, $\waw$, will all be defined based on this constrained LP. 
In the rest of this section, we show that if we make similar realizability assumptions as in Section~\ref{sec:results} but w.r.t.~$\vaw$ and $\waw$ (instead of $\va$ and $\wa$), then we can compete with $\piaw$ without needing to make any coverage assumption on the data distribution $d^D$. 


\begin{samepage}
\begin{problem*}[Constrained $\&$ regularized LP]\label{prob:constrained_lp}
\vspace{-10pt}
\begin{align}
	&\max_{0\le d \le B_wd^D}\mathbb{E}_{(s,a)\sim d}[r(s,a)]-\alpha\mathbb{E}_{(s,a)\sim d^D}\left[f\left(\frac{d(s,a)}{d^D(s,a)}\right)\right]\label{eq:constrained 2}\\
	&\text{s.t. }d(s)=(1-\gamma)\mu_0(s)+\gamma\sum_{s',a'}P(s|s',a')d(s',a')\label{eq:bellman flow 2}
\end{align}
\vspace{-15pt}
\end{problem*}
\end{samepage}
Following a similar argument as the derivation of \mainalg, we can show that Problem~\eqref{eq:constrained 2}
is equivalent to the maximin problem:
\begin{equation}
\label{prob:maximin2 constrained}
\max_{0\leq w\leq B_w}\min_{{v}}{\La}({v},w):=(1-\gamma)\mathbb{E}_{s\sim \mu_0}[{v}(s)]-\alpha\mathbb{E}_{(s,a)\sim d^D}[f(w(s,a))]+\mathbb{E}_{(s,a)\sim d^D}[w(s,a)e_{{v}}(s,a)],
\end{equation}
Denote the optimum of (\ref{prob:maximin2 constrained}) by $({\vaw},{\waw})$, then the optimal policy and its associated discounted state occupancy can be recovered as follows:
\begin{equation}
\label{eq:tilde d tilde w constrained}
{\piaw}(s|a):=
\begin{cases}
\frac{{\waw}(s,a)\pi_D(a|s)}{\sum_a {\waw}(s,a)\pi_D(a|s)}, & \text{for } \sum_a {\waw}(s,a)\pi_D(a|s)>0,\\
\frac{1}{|\mathcal{A}|}, & \text{else.}
\end{cases}, \forall s\in\mathcal{S},a\in\mathcal{A},
\end{equation}
\begin{equation}
{\daw}(s,a)={\waw}(s,a)d^D(s,a).
\end{equation}	

We now state the realizability and boundedness assumptions, which are similar to Section~\ref{sec:results regularized}.
\begin{assumption}[Realizability of ${\VC}$ II]
	\label{ass:V realize 2}
	Suppose ${\vaw}\in {\VC}$. 
\end{assumption}  
\begin{assumption}[Realizability of ${\WC}$ II]
	\label{ass:W realize 2}
	Suppose ${\waw}\in {\WC}$.
\end{assumption} 
\begin{assumption}[Boundedness of ${\WC}$ II]
\label{ass:W bound 2}
Suppose $0\leq w(s,a)\leq B_w$ for any $s\in\mathcal{S}, a\in\mathcal{A}, w\in {\WC}$.
\end{assumption}
\begin{assumption}[Boundedness of $f$ II]
\label{ass:f bound 2}
Suppose that
\begin{align}
&|f'(x)|\leq B_{f'},\forall\quad 0\leq x\leq B_w,\\
&|f(x)|\leq B_f, \forall\quad 0\leq x\leq B_w.
\end{align}
\end{assumption}
Next we consider the boundedness of $\VC$. Similar to Assumption~\ref{ass:V bound}, we will decide the appropriate bound on functions in $\VC$ based on that of $\vaw$, which needs to be captured by $\VC$. It turns out that the additional 
constraint $w\leq B_w$ makes it difficult to derive an upper bound on $\vaw$. However, we are able to do so under a common and mild assumption, that the data distribution $d^D$ is a valid occupancy \citep{liu2018breaking,tang2019doubly,levine2020offline}:
\begin{assumption}
	\label{ass:dataset}
	Suppose $d^D=d^{\pi_D}$, i.e., the discounted occupancy of behavior policy $\pi_D$.
\end{assumption}

With Assumption~\ref{ass:dataset}, we have $\Vert{\vaw}\Vert_{\infty}\leq B_{v}$ from Lemma~\ref{lem:bound tilde nu 2} and therefore the following assumption is reasonable:
\begin{assumption}[Boundedness of ${\VC}$ II]
\label{ass:V bound 2}
Suppose $\Vert{v}\Vert_{\infty}\leq B_{v}:=\frac{\alpha B_{f'}+1}{1-\gamma}$ for any ${v}\in {\VC}$.
\end{assumption}

With the above assumptions, we have the following theorem to show that \mainalg is able to learn $\piaw$:

\begin{theorem}
	\label{thm:sample regularized constrained}
	Assume $\alpha>0$. Suppose \ref{ass:V realize 2},\ref{ass:W realize 2},\ref{ass:W bound 2},\ref{ass:f bound 2},\ref{ass:dataset},\ref{ass:V bound 2} and strong convexity in \ref{ass:f prop} hold. Then with at least probability $1-\delta$, the output of \mainalg satisfies:
	\begin{align}
		&J(\piaw)-J(\hat{\pi})\leq\frac{1}{1-\gamma}\mathbb{E}_{s\sim {\daw}}[\Vert{\piaw}(\cdot|s)-\hat{\pi}(\cdot|s)\Vert_1]\notag\\
		&\leq \frac{2}{1-\gamma}\Vert \hat{w}-{\waw}\Vert_{2,d^D}\leq \frac{4}{1-\gamma}\sqrt{\frac{\SE_{n,n_0,\alpha}(B_w,B_f,B_v,B_e)}{\alpha M_f}}, \label{eq:thm4 1}
	\end{align}
	where $B_{e}:=(1+\gamma)B_{v}+1$.
\end{theorem}

\begin{proof}[Proof sketch]
The proof largely follows Theorem~\ref{thm:sample regularized} except the derivation of the bound on ${\vaw}$, which is characterized in the following lemma:
\begin{lemma}
	\label{lem:bound tilde nu 2}
	Suppose Assumption~\ref{ass:f bound 2} holds, then we have:
	\begin{equation}
	\Vert{\vaw}\Vert_{\infty}\leq B_{v}.
	\end{equation}
\end{lemma}
The proof of Lemma~\ref{lem:bound tilde nu 2} is deferred to Appendix~\ref{proof:lem bound tilde nu 2}. The rest of the proof of Theorem~\ref{thm:sample regularized constrained} is the same as in Section~\ref{sec:analysis} and thus omitted here.
\end{proof}

As before, we obtain the following corollary for competing with the best policy in $\PIW$:
\begin{corollary}
	\label{cor:sample unregularized constrained}
	For any $\epsilon>0$, assume that Assumption~\ref{ass:V realize 2},\ref{ass:W realize 2},\ref{ass:W bound 2},\ref{ass:f bound 2},\ref{ass:dataset},\ref{ass:V bound 2} and strong convexity in \ref{ass:f prop} hold for $\alpha=\alpha'_{\epsilon}:=\frac{\epsilon}{4{B_f}}$. Then if 
	\begin{align}
	&n\geq\frac{C_1\left(\epsilon {B_f}+4{B_w}{B_e}{B_f}\right)^2}{\epsilon^6M_f^2(1-\gamma)^4}\cdot\log\frac{4|\VC||\WC|}{\delta}, \\
	&n_0\geq\frac{C_1\left(4{B_v}{B_f}\right)^2}{\epsilon^6M_f^2(1-\gamma)^2}\cdot\log\frac{4|\VC|}{\delta},
\end{align}
	the output of \mainalg with input $\alpha=\alpha'_{\epsilon}$ satisfies
	\begin{equation}
		J(\pizw)-J(\hat{\pi})\leq\epsilon,
	\end{equation}
	with at least probability $1-\delta$, where $C_1$ is the same constant in Corollary~\ref{cor:sample unregularized}.
\end{corollary}
\begin{proof}
	First notice that 
	\begin{equation}
		\mathbb{E}_{(s,a)\sim {{\dew}}}[r(s,a)]-\alpha'_{\epsilon}\mathbb{E}_{(s,a)\sim d^D}[f({\wew}(s,a))]\geq\mathbb{E}_{(s,a)\sim \dzw}[r(s,a)]-\alpha'_{\epsilon}\mathbb{E}_{(s,a)\sim d^D}[f({\wzw}(s,a))],
	\end{equation}
	which implies that
	\begin{align}
		J({\pizw})-J({\piew})
		&\leq\alpha'_{\epsilon}\left(\mathbb{E}_{(s,a)\sim d^D}[f({\wzw}(s,a))]-\mathbb{E}_{(s,a)\sim d^D}[f({\wew}(s,a))]\right)\\
		&\leq 2\alpha'_{\epsilon}B_f=\frac{\epsilon}{2}.
	\end{align}
On the other hand, by Theorem~\ref{thm:sample regularized constrained} we have with probability at least $1-\delta$,
\begin{equation}
\mathbb{E}_{s\sim {\dew}}[\Vert{\piew}(\cdot|s)-\hat{\pi}(\cdot|s)\Vert_1]\leq\frac{(1-\gamma)\epsilon}{2}.
\end{equation}
Using the performance difference lemma as in Appendix~\ref{proof:cor sample unregularized}, this implies
\begin{equation}
J({\piew})-J(\hat{\pi})\leq\frac{\epsilon}{2}.
\end{equation}
Therefore, we have $J({\pizw})-J(\hat{\pi})\leq\epsilon$ with at least probability $1-\delta$.
\end{proof}
\begin{remark}
Corollary~\ref{cor:sample unregularized constrained} does not need the assumption of non-negativity of $f$. The reason is that we are already considering a bounded space ($0\leq w\leq B_w$) and thus f must be lower bounded in this space.
\end{remark}

\paragraph{Resolving two-policy concentrability of Corollary~\ref{cor:sample unregularized}}
\label{rem:sample unregularized constrained}	
As we have commented below Corollary~\ref{cor:sample unregularized}, to compete with $\piz$ we need ``two-policy'' concentrability, i.e., Assumption~\ref{ass:concentrability} for both $\alpha=0$ and $\alpha= \alpha_\epsilon$. Here we resolve this issue in Corollary~\ref{cor:single policy} below, by invoking Corollary~\ref{cor:sample unregularized constrained} with $B_w$ set to $B_{w,0}$. This way, we obtain the coverage over the regularized optimal policy $\piaw$ (i.e., the counterpart of $\pia$ in Corollary~\ref{cor:sample unregularized}) \textit{for free}, thus only need the concentrability w.r.t.~$\piz$. 
\begin{corollary}
	\label{cor:single policy}
	Suppose there exists $d^*_0\in D^*_0$ that satisfies Assumption~\ref{ass:concentrability} with $\alpha=0$. For any $\epsilon>0$, assume that Assumption~\ref{ass:V realize 2},\ref{ass:W realize 2},\ref{ass:W bound 2},\ref{ass:f bound 2},\ref{ass:dataset},\ref{ass:V bound 2} and strong convexity in \ref{ass:f prop} hold for $B_w=\BWZ$ and $\alpha=\alpha'_{\epsilon}:=\frac{\epsilon}{4{\BFZ}}$. Then if 
	\begin{align}
		&n\geq\frac{C_1\left(\epsilon {\BFZ}+4{\BWZ}{\BEZ}{\BFZ}\right)^2}{\epsilon^6M_f^2(1-\gamma)^4}\cdot\log\frac{4|\VC||\WC|}{\delta}, \\
		&n_0\geq\frac{C_1\left(4{B_{v,0}}{\BFZ}\right)^2}{\epsilon^6M_f^2(1-\gamma)^2}\cdot\log\frac{4|\VC|}{\delta},
	\end{align}
	the output of \mainalg with input $\alpha=\alpha'_{\epsilon}$ satisfies
	\begin{equation}
		J(\piz)-J(\hat{\pi})\leq\epsilon,
	\end{equation}
	with at least probability $1-\delta$, where $C_1$ is the same constant in Corollary~\ref{cor:sample unregularized}.
\end{corollary}
\begin{proof}
Let $B_w=\BWZ$ in Corollary~\ref{cor:sample unregularized constrained}, then we know $\pizw=\piz$ and Corollary~\ref{cor:single policy} follows directly.
\end{proof}
Corollary~\ref{cor:single policy} shows that our algorithm is able to compete with $\piz$ under concentrability with respect to ${\piz}$ alone. In addition, a version of Proposition~\ref{prop:LP stability} applies to Corollary~\ref{cor:sample unregularized constrained}, which indicates that $\wew={\wz}$ for sufficiently small $\epsilon$. 

\begin{remark}
	Corollary~\ref{cor:single policy} still holds when we set $B_w\geq B_{w,0}$ in case $B_{w,0}$ is unknown. However the realizability assumptions will depend on the choice of $B_w$ and change accordingly.
\end{remark}

\subsection{Policy extraction via behavior cloning}
\label{sec:results BC}
In this section we consider an \textit{unknown} behavior policy $\pi_D$. Notice that the only place we require $\pi_D$ in our algorithm is the policy extraction step, where we compute $\hat\pi$ from $\hat w$ using knowledge of $\pi_D$. Inspired by the imitation learning literature \citep{pomerleau1989alvinn,ross2014reinforcement,agarwal2020flambe}, we will use behavior cloning to compute a policy $\bar{\pi}$ to approximate $\hat{\pi}$, where $\hat\pi$ is not directly available and only implicitly defined via $\hat{w}$ and the data. 

As is standard in the literature \citep{ross2014reinforcement,agarwal2020flambe}, we utilize a policy class $\Pi$ to approximate the target policy. We suppose $\Pi$ is realizable:
\begin{assumption}[Realizability of $\Pi$]
	\label{ass:pi realize}
	Assume $\pia\in\Pi$. 
\end{assumption}
One may be tempted to assume $\hat\pi \in \Pi$, since $\hat\pi$ is the target of imitation, but $\hat{\pi}$ is a function of the data and hence random. A standard way of ``determinizing'' such an assumption is to assume the realizability of $\Pi$ for \textit{all possible} $\hat\pi$ that can be induced by any $w\in \WC$, which leads to a prohibitive  ``completeness''-type assumption. Fortunately, as we have seen in previous sections, $\hat{\pi}$ will be close to ${\pia}$ when learning succeeds, so the realizability of $\pia$---a policy whose definition does not depend on data randomness---suffices for our purposes.

In the rest of this section, we design a novel behavior cloning algorithm which is more robust compared to the classic maximum likelihood estimation process \citep{pomerleau1989alvinn,ross2014reinforcement,agarwal2020flambe}. In MLE behavior cloning, the KL divergence between the target policy and the policy class need to be bounded while in our algorithm we only require the weighted $\ell_1$ distance to be bounded. This property is important in our setting, as \mainalg can only guarantee a small weighted $\ell_2$ distance between $\pia$ and $\hat{\pi}$; $\ell_2$ distance is stronger than $\ell_1$ while weaker than KL divergence. 

Our behavior cloning algorithm is inspired by the algorithms in \cite{sun2019provably,agarwal2019reinforcement}, which require access to $d^{\pi}$ for all $\pi\in\Pi$ and  is not satisfied in our setting. However, the idea of estimating total variation by the variational form turns out to be useful. More concretely, for any two policies $\pi$ and $\pi'$, define:
\begin{equation}
	\label{defn:f pi pi'}
	h^s_{\pi,\pi'}\coloneqq\arg\max_{h:\Vert h\Vert_{\infty}\leq 1}[\mathbb{E}_{a\sim\pi(\cdot|s)} h(a)-\mathbb{E}_{a\sim\pi'(\cdot|s)} h(a)].
\end{equation} 
Let $h_{\pi,\pi'}(s,a)=h^s_{\pi,\pi'}(a),\forall s,a$. Note that the function $h_{\pi, \pi'}$ is purely a function of $\pi$ and $\pi'$ and does not depend on the data or the MDP, and hence can be computed exactly even before we see the data. Such a function witnesses the $\ell_1$ distance between $\pi$ and $\pi'$, as shown in the following lemma; see proof in Appendix~\ref{proof:lem variation}:
\begin{lemma}
	\label{lem:variation}
	For any distribution $d$ on $\mathcal{S}$ and policies $\pi,\pi'$ , we have:
	\begin{equation}
		\mathbb{E}_{s\sim d}[\Vert\pi(\cdot|s)-\pi'(\cdot|s)\Vert_1]=\mathbb{E}_{s\sim d}\left[\mathbb{E}_{a\sim\pi(\cdot|s)}[h_{\pi,\pi'}(s,a)]-\mathbb{E}_{a\sim\pi'(\cdot|s)}[h_{\pi,\pi'}(s,a)]\right].
	\end{equation}
\end{lemma}

Inspired by Lemma~\ref{lem:variation}, we can estimate the total variation distance between $\pi$ and $\pi'$ by evaluating $\mathbb{E}_{a\sim\pi(\cdot|s)}[h_{\pi,\pi'}(s,a)]-\mathbb{E}_{a\sim\pi'(\cdot|s)}[h_{\pi,\pi'}(s,a)]$ empirically. Let $\mathcal{H} := \{h_{\pi,\pi'}: \pi, \pi' \in \Pi\}$ and we have $|\mathcal{H}|\leq|\Pi|^2$. We divide $\mathcal{D}$ into $\mathcal{D}_1$ and $\mathcal{D}_2$ where $\mathcal{D}_1$ is utilized for evaluating $\hat{w}$ and $\mathcal{D}_2$ for obtaining $\bar{\pi}$. Let $n_1$ and $n_2$ denote the number of samples in $\mathcal{D}_1$ and $\mathcal{D}_2$. Then our behavior cloning algorithm is based on the following objective function, whose expectation is $\mathbb{E}_{s\sim \hat{d},a\sim\hat{\pi}}[h^{\pi}(s)-h(s,a)]$ and by Lemma~\ref{lem:variation} is exactly the TV between $\hat{\pi}$ and $\pi$:
\begin{equation}
	\label{BC obj}
	\bar{\pi}=\arg\min_{\pi\in\Pi}\max_{h\in\mathcal{H}}[\sum_{i=1}^{n_2}\hat{w}(s_i,a_i)\left(h^{\pi}(s_i)-h(s_i,a_i)\right)],
\end{equation}
where $(s_i,a_i)\in\mathcal{D}_2,\forall 1\leq i\leq n_2$, $h^{\pi}(s)=\mathbb{E}_{a\sim\pi(\cdot|s)}[h(s,a)]$ and $\bar{\pi}$ is the ultimate output policy. 

It can be observed that (\ref{BC obj}) is the importance-sampling version of
\begin{equation}
\mathbb{E}_{s\sim \hat{d}}[\mathbb{E}_{a\sim\pi(\cdot|s)}[h(s,a)]-\mathbb{E}_{a\sim\hat{\pi}(\cdot|s)}[h(s,a)]]. 
\end{equation}
Since $\hat{d}$ is close to $\da$, by minimizing (\ref{BC obj}) we can find a policy that approximately minimizes $\mathbb{E}_{s\sim \da}[\Vert\pi(\cdot|s)-\hat{\pi}(\cdot|s)\Vert_1]$. We call \mainalg with this behavior cloning algorithm by \algbc.

%
%

Theorem~\ref{thm:sample regularized BC} shows that \algbc can attain almost the same sample complexity as \mainalg in Theorem~\ref{thm:sample regularized} where $\pi_D$ is known.

\begin{theorem}[Sample complexity of learning $\pia$ with unknown  behavior policy]
	\label{thm:sample regularized BC}
	Assume $\alpha>0$. Suppose Assumption~\ref{ass:concentrability},\ref{ass:V realize},\ref{ass:W realize},\ref{ass:W bound},\ref{ass:f prop},\ref{ass:V bound} and \ref{ass:pi realize} hold. Then with at least probability $1-\delta$, the output of \algbc satisfies:
	\begin{align}
		&J(\pia)-J(\bar{\pi})\leq\frac{1}{1-\gamma}\mathbb{E}_{s\sim {\da}}[\Vert{\pia}(\cdot|s)-\bar{\pi}(\cdot|s)\Vert_1]\notag\\
		&\leq\frac{4{\BW}}{1-\gamma}\sqrt{\frac{6\log\frac{4|\Pi|}{\delta}}{n_2}}+\frac{50}{1-\gamma}\sqrt{\frac{\SE_{n_1,n_0,\alpha}(\BW,\BF,\BV,\BE)}{\alpha M_f}},
	\end{align}	
	where ${\BE}$ is defined as in Theorem~\ref{thm:sample regularized}.
\end{theorem}
\begin{proof}
	See Appendix~\ref{proof:thm sample regularized BC} for details.
\end{proof}

\begin{remark}
	Notice that the error scales with $O(\frac{1}{\sqrt{n_2}})$ and $O(\frac{1}{n_1^{\frac{1}{4}}})$, which means that the extra samples required by behavior cloning only affects the higher-order terms. Therefore the total sample complexity $n=n_1+n_2$ is dominated by $n_1$, which coincides with the sample complexity of Theorem~\ref{thm:sample regularized}.
\end{remark}


Similarly, behavior cloning can be extended to the unregularized setting where we compete with $\piz$, and the sample complexity will remain almost the same as Corollary~\ref{cor:sample unregularized}:
\begin{corollary}
	\label{cor:sample unregularized BC}
	Fix any $\epsilon>0$. Suppose there exists $d^*_0\in D^*_0$ such that Assumption~\ref{ass:concentrability} holds. Besides, assume that Assumption~\ref{ass:concentrability},\ref{ass:V realize},\ref{ass:W realize},\ref{ass:W bound},\ref{ass:f prop},\ref{ass:V bound} and \ref{ass:pi realize} hold for $\alpha=\alpha_{\epsilon}$. Then if 
	\begin{align}
		&n_0\geq C_2\cdot\frac{\left(2{\BVE}{\BFZ}\right)^2}{\epsilon^6M_f^2(1-\gamma)^2}\cdot\log\frac{4|\VC|}{\delta},\\
		&n_1\geq C_3\cdot\frac{\left(\epsilon \BFE+2\BWE\BEE\BFZ\right)^2}{\epsilon^6M_f^2(1-\gamma)^4}\cdot\log\frac{|{\VC}||{\WC}|}{\delta}, \\
		&n_2\geq C_4\cdot\frac{(\BWE)^2}{(1-\gamma)^2\epsilon^2}\log\frac{|\Pi|}{\delta},
	\end{align}
	where $C_2,C_3,C_4$ are some universal positive constants, the output of \algbc with input $\alpha=\alpha_{\epsilon}$ satisfies
	\begin{equation}
		J(\piz)-J(\bar{\pi})\leq \epsilon,
	\end{equation}
	with at least probability $1-\delta$.
\end{corollary}
\begin{proof}
	The proof is  the same as in Appendix~\ref{proof:cor sample unregularized}. The only difference is that we replace the result in Theorem~\ref{thm:sample regularized} with Theorem~\ref{thm:sample regularized BC}.
\end{proof}
\begin{remark}
	The sample complexity to obtain $\epsilon$-optimal policy is still $\tilde{O}\left(\frac{(\BWZ)^4(\BWE)^2}{\epsilon^6(1-\gamma)^{6}}\right)$ since $n_2$ is negligible compared to $n_1$.
\end{remark}
\begin{remark}
Similar to Corollary~\ref{cor:sample unregularized}, the concentrability assumptions in Corollary~\ref{cor:sample unregularized BC} can be reduced to single-policy concentrability with the help of Corollary~\ref{cor:sample unregularized constrained}.
\end{remark}

\subsection{\mainalg with $\alpha=0$} 
\label{sec:results unregularized}

From the previous discussions, we notice that when $\alpha>0$, extending from regularized problems to unregularized problems will cause worse sample complexity in \mainalg (Remark~\ref{rem:sample regularized},\ref{rem:sample unregularized constrained}). Also, the realizability assumptions are typically with respect to the regularized optimizers rather than the more natural $(\vz,\wz)$. In this section we show that by using stronger concentrability assumptions, \mainalg can still have guarantees with $\alpha=0$ under the realizability w.r.t.~$(\vz,\wz)$ and attain a faster rate. More specifically, we need the following strong concentration assumption:

\begin{assumption}[Strong concentrability]
\label{ass:strong conc}
Suppose the dataset distribution $d^D$ and some $d^*_0\in D^*_0$ satisfy
\begin{align}
&\frac{d^{\pi}(s)}{d^D(s)}\leq {\BS},\forall \pi, s\in\mathcal{S},\label{eq:str conc 2-1}\\
&\frac{{\dz}(s)}{d^D(s)}\geq {\BST}>0,\forall s\in\mathcal{S}.\label{eq:str conc 2-2}
\end{align}
\end{assumption}
\begin{remark}
Eq.~\eqref{eq:str conc 2-1} is the standard all-policy concentrability assumption in offline RL \citep{chen2019information,nachum2019algaedice,xie2020q}. In addition, Assumption~\ref{ass:strong conc} requires the density ratio of the optimal policy is lower bounded, which is related to an ergodicity assumption used in some previous works in the simulator setting~\citep{wang2017primal,wang2020randomized}.
\end{remark}
\begin{remark}
Recall the counterexample in Section~\ref{sec:algorithm}. It can be observed that $B_{w,l}=0$ in that case and thus the counterexample does not satisfy Assumption~\ref{ass:strong conc}.
\end{remark}


In the following discussion $w^*_0$ and $\pi^*_0$ are specified as the optimal density ratio and policy with respect to the $d^*_0$ in Assumption~\ref{ass:strong conc}. We need to impose some constraints on the function class ${\WC}$ and ${\VC}$ so that $d^{\hat \pi} $ can be upper bounded by $\hat w \cdot d^D$. 
\begin{assumption}
\label{ass:W good}
Suppose
\begin{align}
&{\WC}\subseteq\bar{{\WC}}:=\notag\\
&\left\{w(s,a)\geq0, \sum_{a}\pi_D(a|s)w(s,a)\geq {\BST}, \forall s\in\mathcal{S}, a\in\mathcal{A}\right\},
\end{align}
\end{assumption}
Given a function class $\WC$, this assumption is trivially satisfied by removing the $w \in \WC$ that are not in $\bar \WC$ when $\pi_D$ is known.

\begin{assumption}
	\label{ass:V good}
	Suppose
	\begin{align}
    0\leq{v}(s)\leq\frac{1}{1-\gamma},\forall s\in\mathcal{S},{v}\in {\VC}.
	\end{align}
\end{assumption}
By Assumption~\ref{ass:strong conc}, ${\wz}\in\bar{{\WC}}$ and $0\leq{\vz}\leq\frac{1}{1-\gamma}$. Therefore Assumption~\ref{ass:W good} and Assumption~\ref{ass:V good} are reasonable. 

With strong concentrability, we can show that \mainalg with $\alpha=0$ can learn an $\epsilon$-optimal policy with sample complexity $n=\tilde{O}\left(\frac{1}{\epsilon^2}\right)$:
\begin{corollary}
\label{cor:sample unregularized 0}
Suppose Assumption~\ref{ass:concentrability},\ref{ass:V realize},\ref{ass:W realize},\ref{ass:W bound}, \ref{ass:W good}, \ref{ass:V good} and \ref{ass:strong conc} hold for $\alpha=0$. Then with at least probability $1-\delta$, the output of \mainalg with input $\alpha=0$ satisfies:
\begin{equation}
	J({\piz})-J(\hat{\pi})\leq\frac{2{\BWZ}{\BS}}{(1-\gamma){\BST}}\sqrt{\frac{2\log\frac{4|{\VC}||{\WC}|}{\delta}}{n}}+\frac{{\BS}}{{\BST}}\sqrt{\frac{2\log\frac{4|{\VC}|}{\delta}}{n_0}},
\end{equation}
\end{corollary}
\begin{proof}
The key idea is to utilize Lemma~\ref{lem:hat L tilde L} to bound ${\Lz}({\vz},{\wz})-{\Lz}({\vz},\hat{w})$ and then quantify the performance difference $J({\piz})-J(\hat{\pi})$. See Appendix~\ref{proof:cor sample unregularized 0} for details.
\end{proof}

\paragraph{Comparison with $\alpha>0$ and $\alpha=0$.}
When solving the unregularized problem, \mainalg with $\alpha=0$ has better sample complexity than Corollary~\ref{cor:sample unregularized}. Also the realizability assumptions in Corollary~\ref{cor:sample unregularized 0} are with respect to the optimizers of the unregularized problem itself, which is not the case in Corollary~\ref{cor:sample unregularized} when $\epsilon$ is large. However, \mainalg with $\alpha=0$ only works under a very strong concentrability assumption (Assumption~\ref{ass:strong conc}) and thus is less general than \mainalg with $\alpha>0$.

\section{Analysis for regularized offline RL (Theorem~\ref{thm:sample regularized})}
\label{sec:analysis}
In this section we present the analysis for our main result in Theorem~\ref{thm:sample regularized}.

\subsection{Intuition: invariance of saddle points}
\label{sec:intuition}
First we would like to provide an intuitive explanation why optimizing ${\VC}\times {\WC}$ instead of $\mathbb{R}^{|\mathcal{S}|}\times\mathbb{R}^{|\mathcal{S}||\mathcal{A}|}_+$ can still bring us close to $({\va},{\wa})$. More specifically, we have the following lemma:
\begin{lemma}[Invariance of saddle points]
	\label{lem:minimax}
	Suppose $(x^*,y^*)$ is a saddle point of $f(x,y)$ over $\mathcal{X}\times\mathcal{Y}$, then for any $\mathcal{X}'\subseteq\mathcal{X}$ and $\mathcal{Y}'\subseteq\mathcal{Y}$, if $(x^*,y^*)\in\mathcal{X}'\times\mathcal{Y}'$, we have:
	\begin{align}
	&(x^*,y^*)\in\arg\min_{x\in\mathcal{X}'}\arg\max_{y\in\mathcal{Y'}}f(x,y),\\
	&(x^*,y^*)\in\arg\max_{y\in\mathcal{Y'}}\arg\min_{x\in\mathcal{X}'}f(x,y).
	\end{align}
\end{lemma}
\begin{proof}
	See Appendix~\ref{proof:lemma minimax}.
\end{proof}

Lemma~\ref{lem:minimax} shows that as long as a subset includes the saddle point of the original set, the saddle point will still be a minimax and maximin point with respect to the subset. We apply this to \eqref{prob:maximin2}: the saddle point $(\va,\wa)$ of \eqref{prob:maximin2}, also the solution to the regularized MDP without any restriction on function classes, is also a solution of $\max_{w\in {\WC}}\min_{{v}\in {\VC}}{\La}({v},w)$. 

We now give a brief sketch.  Since $\hat L $ is unbiased for $\La$, using uniform convergence, $\hat{L}_{\alpha}({v},w)\approx {\La}({v},w)$ with high probability. Next, use strong concavity of ${\La}({v},w)$ with respect to $w$, to show that $\hat{w}\approx {\wa}$. This implies that  $\hat{\pi}\approx {\pia}$, which is exactly Theorem~\ref{thm:sample regularized}.

\subsection{Preparation: boundedness of ${\va}$}
\label{sec:nu bound}
Before proving Theorem~\ref{thm:sample regularized}, an important ingredient is to bound ${\va}$ since $\VC$ is assumed to be a bounded set (Assumption \ref{ass:V bound}). The key idea is to utilize KKT conditions and the fact that for each $s\in\mathcal{S}$ there exists $a\in\mathcal{A}$ such that ${\wa}(s,a)>0$. The consequent bound is given in Lemma~\ref{lem:bound tilde nu}.
\begin{lemma}[Boundedness of ${\va}$]
\label{lem:bound tilde nu}
Suppose Assumption~\ref{ass:concentrability} and \ref{ass:f prop} holds, then we have:
\begin{equation}
\Vert{\va}\Vert_{\infty}\leq {\BV}:=\frac{\alpha {\BFP}+1}{1-\gamma}.
\end{equation}
\end{lemma}
\begin{proof}
See Appendix~\ref{proof:lemma bound tilde nu}.
\end{proof}

\subsection{Proof sketch of Theorem~\ref{thm:sample regularized}}
As stated in Section~\ref{sec:intuition}, our proof consists of (1) using concentration inequalities to bound $|{\La}({v},w)-\hat{L}_{\alpha}({v},w)|$, (2) using the invariance of saddle points and concentration bounds to characterize the error $\Vert\hat{w}-{\wa}\Vert_{2,d^D}$ and (3) analyzing the difference between $\hat{\pi}$ and ${\pia}$. We will elaborate on each of these steps in this section.
\paragraph{Concentration of $\hat{L}_{\alpha}({v},w)$.}
First, it can be observed that $\hat{L}_{\alpha}({v},w)$ is an unbiased estimator of ${\La}({v},w)$, as shown in the following lemma
\begin{lemma}
\label{lem:unbiased hat L}
\begin{equation}
\mathbb{E}_{\mathcal{D}}[\hat{L}_{\alpha}({v},w)]={\La}({v},w),\quad\forall{v}\in {\VC},w\in {\WC},
\end{equation}
where $\mathbb{E}_{\mathcal{D}}[\cdot]$ is the expectation with respect to the samples in $\mathcal{D}$, i.e., $(s_i,a_i)\sim d^D, s'_i\sim P(\cdot|s_i,a_i)$.
\end{lemma}
\begin{proof}
See Appendix~\ref{proof:lemma unbiased hat L}.
\end{proof}

On the other hand, note that from the boundedness of ${\VC},{\WC}$ and $f$ (Assumption~\ref{ass:V bound}, \ref{ass:W bound}, \ref{ass:f prop}), $\hat{L}_{\alpha}({v},w)$ is also bounded. Combining with Lemma~\ref{lem:unbiased hat L}, we have the following lemma: 
\begin{lemma}
\label{lem:hat L conc}
Suppose Assumption~\ref{ass:W bound},\ref{ass:f prop},\ref{ass:V bound} hold. Then with at least probability $1-\delta$, for all ${v}\in {\VC}$ and $w\in {\WC}$ we have:
\begin{equation}
	|\hat{L}_{\alpha}({v},w)-{\La}({v},w)|\leq\SE_{n,n_0,\alpha}(\BW,\BF,\BV,\BE):={\lstat},
\end{equation}
\end{lemma}
\begin{proof}
See Appendix~\ref{proof:hat L conc}.
\end{proof}

\paragraph{Bounding $\Vert\hat{w}-{\wa}\Vert_{2,d^D}$.}
To bound $\Vert\hat{w}-{\wa}\Vert_{2,d^D}$, we first need to characterize ${\La}({\va},{\wa})-{\La}({\va},\hat{w})$. Inspired by Lemma~\ref{lem:minimax}, we decompose ${\La}({\va},{\wa})-{\La}({\va},\hat{w})$ carefully and utilize the concentration results Lemma~\ref{lem:hat L conc}, which leads us to the following lemma:
\begin{lemma}
\label{lem:hat L tilde L}
Suppose Assumption~\ref{ass:concentrability},\ref{ass:V realize},\ref{ass:W realize},\ref{ass:W bound},\ref{ass:f prop} and \ref{ass:V bound} hold. Then with at least probability $1-\delta$,
\begin{equation}
{\La}({\va},{\wa})-{\La}({\va},\hat{w})\leq 2{\lstat}.
\end{equation}
\end{lemma}
\begin{proof}
See Appendix~\ref{proof:hat L tilde L}.
\end{proof}

Then due to the strong convexity of $f$ which leads to $L_\alpha$ being strongly concave in $w$, $\Vert\hat{w}-{\wa}\Vert_{2,d^D}$ can be naturally bounded by Lemma~\ref{lem:hat L tilde L},
\begin{lemma}
\label{lem: hat w error}
Suppose Assumption~\ref{ass:concentrability},\ref{ass:V realize},\ref{ass:W realize},\ref{ass:W bound},\ref{ass:f prop},\ref{ass:V bound} hold. Then with at least probability $1-\delta$,
\begin{equation}
\label{eq:lem hat w 1}
\Vert\hat{w}-{\wa}\Vert_{2,d^D}\leq \sqrt{\frac{4{\lstat}}{\alpha M_f}},
\end{equation}
which implies that
\begin{equation}
\label{eq:lem hat w 2}
\Vert \hat{d}-{\da}\Vert_{1}\leq \sqrt{\frac{4{\lstat}}{\alpha M_f}},
\end{equation}
where $\hat{d}(s,a)=\hat{w}(s,a)d^D(s,a),\forall s,a$.
\end{lemma}
\begin{proof}
See Appendix~\ref{proof:hat w error}.
\end{proof}
This proves the third part of (\ref{eq:thm1 1}) in Theorem~\ref{thm:sample regularized}. 

\paragraph{Bounding $\mathbb{E}_{s\sim {\da}}[\Vert{\pia}(s,\cdot)-\hat{\pi}(s,\cdot)\Vert_1]$.}
To obtain the second part of (\ref{eq:thm1 1}), we notice that ${\pia}$ (or $\hat{\pi}$) can be derived explicitly from ${\wa}$ (or $\hat{w}$) by (\ref{eq:tilde d tilde w}) (or (\ref{eq:hat pi})). However, the mapping ${\wa}\mapsto{\pia}$ (or $\hat{w}\mapsto\hat{\pi}$) is not linear and discontinuous when ${\da}(s)=0$ (or $\hat{d}(s)=0$), which makes the mapping complicated. To tackle with this problem, we first decompose the error $\Vert\hat{w}-{\wa}\Vert_{2,d^D}$ and assign to each state $s\in\mathcal{S}$, then consider the case where $\hat{d}(s)>0$ and $\hat{d}(s)=0$ separately. Consequently, we can obtain the following lemma:
\begin{lemma}
\label{lem:tilde w tilde pi}
\begin{equation}
\label{eq:tilde w tilde pi}
\mathbb{E}_{s\sim {\da}}[\Vert{\pia}(s,\cdot)-\hat{\pi}(s,\cdot)\Vert_1]\leq2\Vert\hat{w}-{\wa}\Vert_{2,d^D}.
\end{equation}
\end{lemma}
\begin{proof}
See Appendix~\ref{proof:tilde w tilde pi}.
\end{proof}

Combining Equation \eqref{eq:lem hat w 1}, \eqref{eq:tilde w tilde pi}, and the definition of ${\lstat} $ from Lemma \ref{lem:hat L conc}, gives us the second part of Theorem \ref{thm:sample regularized}.


\paragraph{Bounding $J(\pia)-J(\hat{\pi})$.}
To complete the proof of Theorem~\ref{thm:sample regularized}, we only need to bound $J(\pia)-J(\hat{\pi})$ via the bounds on $\mathbb{E}_{s\sim {\da}}[\Vert{\pia}(s,\cdot)-\hat{\pi}(s,\cdot)\Vert_1]$, which is shown in the following lemma:
\begin{lemma}
	\label{lem:tilde pi performance}
	\begin{equation}
		\label{eq:tilde pi performance}
		J(\pia)-J(\hat{\pi})\leq\frac{1}{1-\gamma}\mathbb{E}_{s\sim {\da}}[\Vert{\pia}(s,\cdot)-\hat{\pi}(s,\cdot)\Vert_1].
	\end{equation}
\end{lemma}
\begin{proof}
	See Appendix~\ref{proof:tilde pi performance}.
\end{proof}
This concludes the proof of Theorem \ref{thm:sample regularized}.




%
%
\ifdefined\isarxivversion
\bibliographystyle{apalike}
\else
\fi
\bibliography{ref.bib, RL.bib}
\appendix
\section{Discussion}

\subsection{Comparison with OptiDICE \citep{lee2021optidice}}

Our algorithm is inspired by OptiDICE \citep{lee2021optidice}, but with several crucial modifications necessary to obtain the desired sample-complexity guarantees. 
OptiDICE starts with the problem of 
$\min_{{v}}\max_{w\geq0}{\La}({v},w)$, and 
then uses the closed-form maximizer ${\wa}({v}):=\arg\max_{w\geq0}{\La}({v},w)$ for arbitrary ${v}$ \citep[Proposition 1]{lee2021optidice}:
\begin{equation}
	{\wa}({v})=\max\left(0,(f')^{-1}\left(\frac{e_{{v}}(s,a)}{\alpha}\right)\right),
\end{equation}
and then solves
$\min_{{v}}{\La}({v},{\wa}({v}))$. 
Unfortunately, the $e_v(s,a)$ term in the expression requires  knowledge of the transition function $P$, causing the infamous double-sampling difficulty \citep{baird1995residual, farahmand2011model}, a major obstacle in offline RL with only realizability assumptions \citep{chen2019information}. OptiDICE deals with this by optimizing an upper bound of  $\max_{w\geq0}\La({v},w)$ which does not lend itself to theoretical analysis.  Alternatively, one can fit $e_v$ using a separate function class.  
However, since ${v}$ is arbitrary in the optimization, the function class needs to approximate $e_v$ for all ${v}$, requiring a completeness-type assumption in theory~\citep{xie2020q}. 
In contrast, \mainalg 
optimizes over ${\VC}\times {\WC}$ and thus $\arg\max_{w\in {\WC}}{\La}({v},w)$ is naturally contained in ${\WC}$, and our analyses show that this circumvents the completeness-type assumptions and only requires realizability. 

Another important difference is the policy extraction step. OptiDICE uses a heuristic behavior cloning algorithm without any guarantees. We develop a new behavior cloning algorithm that only requires realizability of the policy and does not increase the sample complexity.

\subsection{Discussion about Assumption~\ref{ass:strong conc}}

\label{sec:discussion cor2 concentration}
The following ergodicity assumption has been introduced in some online reinforcement learning works \citep{wang2017primal,wang2020randomized}:
\begin{assumption}
\label{ass:ergodicity}
Assume 
\begin{equation}
B_{\texttt{erg},1}\mu_0(s)\leq d^{\pi}(s)\leq B_{\texttt{erg},2}\mu_0(s),\forall s,\pi.
\end{equation}
\end{assumption}
\begin{remark}
The original definition of ergodicity in \citet{wang2017primal,wang2020randomized} is targeted at the stationary distribution induced by policy $\pi$ rather than the discounted visitation distribution. However, this is not an essential difference and it can be shown that Corollary~\ref{cor:sample unregularized 0} still holds under the definition in \citet{wang2017primal,wang2020randomized}. Here we define ergodicity with respect to the discounted visitation distribution  for the purpose of comparing Assumption~\ref{ass:strong conc} and \ref{ass:ergodicity}.
\end{remark}

In fact, our Assumption~\ref{ass:strong conc} is weaker than Assumption~\ref{ass:ergodicity} as shown in the following lemma:
\begin{lemma}
\label{lem:ergodic}
Suppose $d^{\pi}(s)\leq B_{\texttt{erg},2}\mu_0(s),\forall s,\pi$ and Assumption~\ref{ass:dataset} holds, then we have:
\begin{align}
&\frac{d^{\pi}(s)}{d^D(s)}\leq\frac{B_{\texttt{erg},2}}{1-\gamma},\forall \pi,s\\
&\frac{\dz(s)}{d^D(s)}\geq\frac{1-\gamma}{B_{\texttt{erg},2}},\forall s.
\end{align}	
\end{lemma}
The proof is deferred to Appendix~\ref{proof:lem ergodic}. Lemma~\ref{lem:ergodic} shows that the upper bound in Assumption~\ref{ass:ergodicity} implies Assumption~\ref{ass:strong conc}. Therefore, our strong concentration assumption is a weaker version of the ergodicity assumption.

\subsection{Combination of different practical factors}
In Section~\ref{sec:results}, we generalized \mainalg to several more realistic settings (approximation and optimization error, poor coverage, unknown behavior policy). In fact, \mainalg with $\alpha>0$ can even be generalized to include all of the three settings by combining Theorem~\ref{thm:sample regularized},\ref{thm:sample regularized approximate},\ref{thm:sample regularized constrained},\ref{thm:sample regularized BC} and Corollaries ~\ref{cor:sample unregularized},\ref{cor:sample unregularized approximate},\ref{cor:sample unregularized constrained},\ref{cor:sample unregularized BC}. For brevity, we do not list all the combinations separately and only illustrate how to handle each individually.



For \mainalg with $\alpha=0$, it is easy to extend Corollary~\ref{cor:sample unregularized 0} to approximation and optimization error but relaxation of the concentration assumption and unknown behavior policy is difficult. This is because the analysis of Corollary~\ref{cor:sample unregularized 0} relies on the fact that $\vz$ is the optimal value function of the unregularized problem~\eqref{prob:original problem}. Consequently, the same analysis is not applicable to $(\wzw,\vzw)$. Furthermore, Assumption~\ref{ass:W good} requires knowing $\pi_D$ and thus hard to enforce with unknown behavior policy.

\section{Proofs of Lemmas for Theorem~\ref{thm:sample regularized}}
\subsection{Proof of Lemma~\ref{lem:minimax}}
\label{proof:lemma minimax}
We first prove that $(x^*,y^*)\in\arg\min_{x\in\mathcal{X}'}\arg\max_{y\in\mathcal{Y'}}f(x,y)$.  Since $(x^*,y^*)$ is a saddle point \citep{sion1958general}, we have
\begin{equation}
\label{eq:nash}
x^*=\arg\min_{x\in\mathcal{X}}f(x,y^*),y^*=\arg\max_{y\in\mathcal{Y}}f(x^*,y).
\end{equation}
Since $\mathcal{Y}'\subseteq\mathcal{Y}$ and $y^*\in\mathcal{Y}'$, we have:
\begin{equation}
f(x^*,y^*)=\max_{y\in\mathcal{Y}'}f(x^*,y).
\end{equation}
On the other hand, because $\mathcal{X}'\subseteq\mathcal{X}$ and $y^*\in\mathcal{Y}'$,
\begin{equation}
f(x^*,y^*)\leq f(x,y^*)\leq \max_{y\in\mathcal{Y}'}f(x,y), \forall x\in\mathcal{X}'.
\end{equation}
Notice that $x^*\in\mathcal{X}'$, so we have:
\begin{equation}
\max_{y\in\mathcal{Y}'}f(x^*,y)=\min_{x\in\mathcal{X}}\max_{y\in\mathcal{Y}'}f(x,y),
\end{equation}
or equivalently,
\begin{equation}
\label{eq:lem1 eq1}
(x^*,y^*)\in\arg\min_{x\in\mathcal{X}'}\arg\max_{y\in\mathcal{Y'}}f(x,y).
\end{equation}

On the other hand, by a similar proof we have 
\begin{equation}
f(x^*,y^*)\geq f(x^*,y)\geq\min_{x\in\mathcal{X}'}f(x,y),\quad\forall y\in\mathcal{Y'},
\end{equation}
which implies that
\begin{equation}
(x^*,y^*)\in\arg\max_{y\in\mathcal{Y'}}\arg\min_{x\in\mathcal{X}'}f(x,y).
\end{equation}

\subsection{Proof of Lemma~\ref{lem:bound tilde nu}}
\label{proof:lemma bound tilde nu}
From the strong duality of the regularized problem (\ref{eq:constrained})(\ref{eq:bellman flow 1}), when $d^D(s,a)\neq0$, we have  ${\wa}=\arg\max_{w\geq 0}{\La}({\va},w)$, or 
\begin{equation}
\label{eq:lem2 eq1}
{\wa}(s,a)=\max\left(0,(f')^{-1}\left(\frac{e_{{\va}}(s,a)}{\alpha}\right)\right).
\end{equation}

Note that ${\da}(s,a)={\wa}(s,a)d^D(s,a)$ satisfies Bellman flow constraint (\ref{eq:bellman flow 1}), therefore
\begin{equation}
{\da}(s)\geq(1-\gamma)\mu_0(s)>0,\quad\forall s\in\mathcal{S},
\end{equation}
which implies that for any $s\in\mathcal{S}$, $\exists a_s\in\mathcal{A}$ such that
\begin{equation}
{\da}(s,a_s)>0,
\end{equation}
or equivalently
\begin{equation}
{\wa}(s,a_s)>0, d^D(s,a_s)>0.
\end{equation}
Thus from (\ref{eq:lem2 eq1}) we know that
\begin{equation}
e_{{\va}}(s,a_s)=\alpha f'({\wa}(s,a_s)).
\end{equation}
From Assumption~\ref{ass:concentrability}, ${\wa}(s,a_s)\leq {\BW}$ and thus due to Assumption~\ref{ass:f prop}, 
\begin{equation}
\label{eq:lem2 eq2}
|e_{{\va}}(s,a_s)|\leq \alpha {\BFP}, \forall s\in\mathcal{S}. 
\end{equation}

On the other hand, suppose $|{\va}(s_m)|=\Vert {\va}\Vert_{\infty}$, then from the definition of $e_{{v}}$ we have:
\begin{equation}
e_{{\va}}(s_m,a_{s_m})=r(s_m,a_{s_m})+\gamma\mathbb{E}_{s'\sim P(\cdot|s_m,a_{s_m})}{\va}(s')-{\va}(s_m),
\end{equation}
which implies that:
\begin{align}
|e_{{\va}}(s_m,a_{s_m})-r(s_m,a_{s_m})|&=|{\va}(s_m)-\gamma\mathbb{E}_{s'\sim P(\cdot|s_m,a_{s_m})}{\va}(s')|\\
&\geq|{\va}(s_m)|-\gamma|\mathbb{E}_{s'\sim P(\cdot|s_m,a_{s_m})}{\va}(s')|\\
&\geq|{\va}(s_m)|-\gamma\mathbb{E}_{s'\sim P(\cdot|s_m,a_{s_m})}|{\va}(s')|\\
&\geq(1-\gamma)|{\va}(s_m)|\label{eq:lem2 eq3}.
\end{align}

Combining (\ref{eq:lem2 eq2}) and (\ref{eq:lem2 eq3}), we have
\begin{equation}
\Vert{\va}\Vert_{\infty}\leq\frac{\alpha {\BFP}+1}{1-\gamma}.
\end{equation}

\subsection{Proof of Lemma~\ref{lem:unbiased hat L}}
\label{proof:lemma unbiased hat L}
First by the tower rule, we have:
\begin{equation}
\mathbb{E}_{\mathcal{D}}\left[\hat{L}_{\alpha}({v},w)\right]=\mathbb{E}_{(s_i,a_i)\sim d^D, s_{0,j}\sim\mu_0}\left[\mathbb{E}_{s'_i\sim P(\cdot|s_i,a_i)}\left[\hat{L}_{\alpha}({v},w)\vert s_i,a_i\right]\right].
\end{equation}
Note that
\begin{align}
&\mathbb{E}_{s'_i\sim P(\cdot|s_i,a_i)}\left[\hat{L}_{\alpha}({v},w)\vert s_i,a_i\right]\\
=&(1-\gamma)\frac{1}{n_0}\sum_{j=1}^{n_0}[{v}(s_{0,j})]+\frac{1}{n}\sum_{i=1}^n[-\alpha f(w(s_i,a_i))]\\+&\frac{1}{n}\sum_{i=1}^n[w(s_i,a_i)\mathbb{E}_{s'_i\sim P(\cdot|s_i,a_i)}\left[e_{{v}}(s_i,a_i,r_i,s'_i)\vert s_i,a_i\right]]\\
=&(1-\gamma)\frac{1}{n_0}\sum_{j=1}^{n_0}[{v}(s_{0,j})]+\frac{1}{n}\sum_{i=1}^n[-\alpha f(w(s_i,a_i))]+\frac{1}{n}\sum_{i=1}^n[w(s_i,a_i)e_{{v}}(s_i,a_i)].
\end{align}
Therefore,
\begin{align}
&\mathbb{E}_{\mathcal{D}}\left[\hat{L}_{\alpha}({v},w)\right]\\
=&(1-\gamma)\mathbb{E}_{s\sim \mu_0}[{v}(s)]-\alpha\mathbb{E}_{(s,a)\sim d^D}[f(w(s,a))]+\mathbb{E}_{(s,a)\sim d^D}[w(s,a)e_{{v}}(s,a)]\\
=&{\La}({v},w).
\end{align}


\subsection{Proof of Lemma~\ref{lem:hat L conc}}
\label{proof:hat L conc}
Let $l^{{v},w}_{i}=-\alpha f(w(s_i,a_i))+w(s_i,a_i)e_{{v}}(s_i,a_i,r_i,s'_i)$. From Assumption~\ref{ass:V bound}, we know
\begin{equation}
|e_{{v}}(s,a,r,s')|=|r(s,a)+\gamma{v}(s')-{v}(s)|\leq(1+\gamma){\BV}+1={\BE}.
\end{equation}
Therefore, by Assumption~\ref{ass:W bound} and \ref{ass:f prop}, we have:
\begin{equation}
|l^{{v},w}_{i}|\leq \alpha {\BF}+{\BW}{\BE}.
\end{equation}

Notice that $l^{{v},w}_i$ is independent from each other, thus we can apply Hoeffding's inequality and for any $t>0$,
\begin{equation}
\text{Pr}[|\frac{1}{n}\sum_{i=1}^nl^{{v},w}_i-\mathbb{E}[l^{{v},w}_i]|\leq t]\geq 1-2\exp\left(\frac{-nt^2}{2(\alpha {\BF}+{\BW}{\BE})^2}\right).
\end{equation} 

Let $t=(\alpha {\BF}+{\BW}{\BE})\sqrt{\frac{2\log\frac{4|{\VC}||{\WC}|}{\delta}}{n}}$, we have with at least probability $1-\frac{\delta}{2|{\VC}||{\WC}|}$,
\begin{equation}
|\frac{1}{n}\sum_{i=1}^nl^{{v},w}_i-\mathbb{E}[l^{{v},w}_i]|\leq(\alpha {\BF}+{\BW}{\BE})\sqrt{\frac{2\log\frac{4|{\VC}||{\WC}|}{\delta}}{n}}.
\end{equation}
Therefore by union bound, with at least probability $1-\frac{\delta}{2}$, we have for all ${v}\in {\VC}$ and $w\in {\WC}$, 
\begin{equation}
|\frac{1}{n}\sum_{i=1}^nl^{{v},w}_i-\mathbb{E}[l^{{v},w}_i]|\leq (\alpha {\BF}+{\BW}{\BE})\sqrt{\frac{2\log\frac{4|{\VC}||{\WC}|}{\delta}}{n}}.
\end{equation}

Similarly, we have  with at least probability $1-\frac{\delta}{2}$, for all ${v}\in {\VC}$, 
\begin{equation}
	|\frac{1}{n_0}\sum_{j=1}^{n_0}v(s_{0,j})-\mathbb{E}_{s\sim\mu_0}[v(s)]|\leq \BV\sqrt{\frac{2\log\frac{4|{\VC}|}{\delta}}{n_0}}.
\end{equation}
Therefore, with at least probability $1-\delta$ we have
\begin{equation}
	|\hat{L}_{\alpha}({v},w)-{\La}({v},w)|\leq(\alpha {\BF}+{\BW}{\BE})\sqrt{\frac{2\log\frac{4|{\VC}||{\WC}|}{\delta}}{n}}+(1-\gamma)\BV\sqrt{\frac{2\log\frac{4|{\VC}|}{\delta}}{n_0}}.
\end{equation}

\subsection{Proof of Lemma~\ref{lem:hat L tilde L}}
\label{proof:hat L tilde L}
First we decompose ${\La}({\va},\hat{w})-{\La}({\va},{\wa})$ into the following terms:
\begin{align}
&{\La}({\va},\hat{w})-{\La}({\va},{\wa})=(\underbrace{{\La}({\va},\hat{w})-\hat{L}_{\alpha}({\va},\hat{w})}_{(1)}) +(\underbrace{\hat{L}_{\alpha}({\va},\hat{w})-\hat{L}_{\alpha}(\hat{{v}},\hat{w})}_{(2)}) \notag\\
&+(\underbrace{\hat{L}_{\alpha}(\hat{{v}},\hat{w})-\hat{L}_{\alpha}(\hat{{v}}({\wa}),{\wa})}_{(3)}) + (\underbrace{\hat{L}_{\alpha}(\hat{{v}}({\wa}),{\wa})-{\La}(\hat{{v}}({\wa}),{\wa})}_{(4)})\\& + (\underbrace{{\La}(\hat{{v}}({\wa}),{\wa})-{\La}({\va},{\wa})}_{(5)}),
\end{align}
where $\hat{{v}}(w)=\arg\min_{{v}\in {\VC}}\hat{L}_{\alpha}({v},w)$.

For term (1) and (4), we can apply Lemma~\ref{lem:hat L conc} and thus
\begin{equation}
(1)\geq-{\lstat},(4)\geq-{\lstat}.
\end{equation}
For term (2), since $\hat{{v}}=\arg\min_{{v}\in {\VC}}\hat{L}_{\alpha}({v},\hat{w})$ and ${\va}\in {\VC}$, we have
\begin{equation}
(2)\geq0.
\end{equation}
For term (3), since $\hat{w}=\arg\max_{w\in {\WC}}\hat{L}_{\alpha}(\hat{{v}}(w),w)$ and ${\wa}\in {\WC}$, 
\begin{equation}
(3)\geq0.
\end{equation}
For term (5), note that due to the strong duality of the regularized problem~(\ref{eq:constrained})(\ref{eq:bellman flow 1}), $({\va},{\wa})$ is a saddle point of ${\La}({v},w)$ over $\mathbb{R}^{|\mathcal{S}|}\times\mathbb{R}_+^{|\mathcal{S}||\mathcal{A}|}$. Therefore,
\begin{equation}
{\va}=\arg\min_{{v}\in\mathbb{R}^{|\mathcal{S}|}}{\La}({v},{\wa}).
\end{equation}
Since $\hat{{v}}({\wa})\in\mathbb{R}^{|\mathcal{S}|}$, we have:
\begin{equation}
(5)\geq0.
\end{equation}

Combining the above inequalities, it is obvious that
\begin{equation}
{\La}({\va},\hat{w})-{\La}({\va},{\wa})\geq-2{\lstat}.
\end{equation}

\subsection{Proof of Lemma~\ref{lem: hat w error}}
\label{proof:hat w error}
First we need to show ${\La}({\va},w)$ is $\alpha M_f$-strongly-concave with respect to $w$ and $\Vert\cdot\Vert_{2,d^D}$. Consider ${\Lp}(w)={\La}({\va},w)+\frac{\alpha M_f}{2}\Vert w\Vert_{2,d^D}^2$, then we know that
\begin{equation}
{\Lp}(w)=(1-\gamma)\mathbb{E}_{s\sim \mu_0}[{v}(s)]-\alpha\mathbb{E}_{(s,a)\sim d^D}[f(w(s,a))-\frac{M_f}{2}w(s,a)^2]+\mathbb{E}_{(s,a)\sim d^D}[w(s,a)e_{{v}}(s,a)].
\end{equation}
Since $f$ is $M_f$-strongly-convex, we know ${\Lp}(w)$ is concave, which implies that ${\La}({\va},w)$ is $\alpha M_f$-strongly-concave with respect to $w$ and $\Vert\cdot\Vert_{2,d^D}$. 

On the other hand, since $({\va},{\wa})$ is a saddle point of ${\La}({v},w)$ over $\mathbb{R}^{|\mathcal{S}|}\times\mathbb{R}_+^{|\mathcal{S}||\mathcal{A}|}$, we have ${\wa}=\arg\max_{w\geq0}{\La}({\va},w)$. Then we have:
\begin{equation}
\Vert\hat{w}-{\wa}\Vert_{2,d^D}\leq \sqrt{\frac{2({\La}({\va},{\wa})-{\La}({\va},\hat{w}))}{\alpha M_f}}.
\end{equation}
Substituting Lemma~\ref{lem:hat L tilde L} into the above equation we can obtain (\ref{eq:lem hat w 1}). For (\ref{eq:lem hat w 2}), it can be observed that
\begin{equation}
\Vert \hat{d}-{\da}\Vert_{1}=\Vert \hat{w}-{\wa}\Vert_{1,d^D}\leq\Vert\hat{w}-{\wa}\Vert_{2,d^D}\leq\sqrt{\frac{4{\lstat}}{\alpha M_f}}.
\end{equation}

\subsection{Proof of Lemma~\ref{lem:tilde w tilde pi}}
\label{proof:tilde w tilde pi}
First note that $\Vert \hat{w}-{\wa}\Vert_{1,d^D}\leq\Vert\hat{w}-{\wa}\Vert_{2,d^D}$, which implies that
\begin{equation}
\sum_s\epsilon_{\hat{w},s}\leq\Vert\hat{w}-{\wa}\Vert_{2,d^D}
\end{equation}
where
\begin{equation} 
\epsilon_{\hat{w},s}=\sum_{a}|\hat{w}(s,a)d^D(s,a)-{\wa}d^D(s,a)|
\end{equation}

If $\hat{d}(s)>0$, then we have:
\begin{align}
&{\da}(s)\sum_a|\hat{\pi}(s,a)-{\pia}(s,a)|\\
=&\sum_a|\frac{{\da}(s)}{\hat{d}(s)}\hat{w}(s,a)d^D(s,a)-{\wa}d^D(s,a)|\label{eq:importance_weight}\\
\leq&\sum_a(|\frac{{\da}(s)}{\hat{d}(s)}-1|\hat{w}(s,a)d^D(s,a))+\sum_a|\hat{w}(s,a)d^D(s,a)-{\wa}d^D(s,a)|\\
\leq&\epsilon_{\hat{w},s}+\sum_a(|\frac{{\da}(s)}{\hat{d}(s)}-1|\hat{w}(s,a)d^D(s,a)).
\end{align}

Notice that $|\hat{d}(s)-{\da}(s)|\leq\epsilon_{\hat{w},s}$, which implies $|\frac{{\da}(s)}{\hat{d}(s)}-1|\leq\frac{\epsilon_{\hat{w},s}}{\hat{d}(s)}$, therefore:
\begin{equation}
{\da}(s)\sum_a|\hat{\pi}(s,a)-{\pia}(s,a)|\leq\epsilon_{\hat{w},s}(1+\sum_a\frac{\hat{w}(s,a)d^D(s,a)}{\hat{d}(s)})=2\epsilon_{\hat{w},s}.
\end{equation}

If $\hat{d}(s)=0$, then we know that $\sum_{a}|{\wa}(s,a)d^D(s,a)|\leq\epsilon_{\hat{w},s}$. Therefore
\begin{align}
{\da}(s)\sum_a|\hat{\pi}(s,a)-{\pia}(s,a)|\leq 2{\da}(s)=2\epsilon_{\hat{w},s}.
\end{align}

Thus we have ${\da}(s)\sum_a|\hat{\pi}(s,a)-{\pia}(s,a)|\leq 2\epsilon_{\hat{w},s}$, from which we can easily obtain:
\begin{equation}
\mathbb{E}_{s\sim {\da}}[\Vert{\pia}(s,\cdot)-\hat{\pi}(s,\cdot)\Vert_1]\leq2\sum_{s}\epsilon_{\hat{w},s}\leq2\Vert\hat{w}-{\wa}\Vert_{2,d^D}.
\end{equation}

\subsection{Proof of Lemma~\ref{lem:tilde pi performance}}
\label{proof:tilde pi performance}
To bound $J({\pia})-J(\hat{\pi})$, we introduce the performance difference lemma which was previously derived in \citet{Kakade02approximatelyoptimal,kakade2003sample}:
\begin{lemma}[Performance Difference]
	\label{lem:performance difference}
	For arbitrary policies $\pi,\pi'$ and initial distribution $\mu_0$, we have
	\begin{equation}
		V^{\pi'}(\mu_0)-V^{\pi}(\mu_0)=\frac{1}{1-\gamma}\mathbb{E}_{s\sim d^{\pi'}}[\langle Q^{\pi}(s,\dot), \pi'(\cdot|s)-\pi(\cdot|s)\rangle].
	\end{equation}
\end{lemma}
The proof of Lemma~\ref{lem:performance difference} is referred to Appendix~\ref{proof:lem performance difference}. With Lemma~\ref{lem:performance difference}, we have
\begin{align}
	&J({\pia})-J(\hat{\pi})\\
	=&(1-\gamma)(V^{{\pia}}(\mu_0)-V^{\hat{\pi}}(\mu_0))\\
	=&\mathbb{E}_{s\sim {\da}}[\langle Q^{\hat{\pi}}(s,\dot), {\pia}(\cdot|s)-\hat{\pi}(\cdot|s)\rangle]\\
	\leq&\frac{1}{1-\gamma}\mathbb{E}_{s\sim {\da}}[\Vert{\pia}(s,\cdot)-\hat{\pi}(s,\cdot)\Vert_1].
\end{align}

\subsection{Proof of Lemma~\ref{lem:performance difference}}
\label{proof:lem performance difference}
For any two policies $\pi'$ and $\pi$, it follows from the definition of $V^{\pi'}(\mu_0)$ that
\begin{align}
	&V^{\pi'}(\mu_0)-V^{\pi}(\mu_0)\\
	=&\mathbb{E}_{\pi'}\left[\sum_{t=0}^{\infty}\gamma^tr(s_t,a_t)\Big\vert\, s_0\sim \mu_0\right]-V^{\pi}(\mu_0) \notag\\
	=&\mathbb{E}_{\pi'}\left[\sum_{t=0}^{\infty}\gamma^t\Big[r(s_t,a_t)+V^{\pi}_{\tau}(s_t)-V^{\pi}(s_t)\Big] \,\Big\vert\, s_0\sim\mu_0\right]-V^{\pi}(\mu_0) \notag\\
	=&\mathbb{E}_{\pi'}\left[\sum_{t=0}^{\infty}\gamma^t\Big[r(s_t,a_t)+\gamma V^{\pi}(s_{t+1})-V^{\pi}(s_t)\Big] \,\Big\vert\, s_0\sim\mu_0\right]\notag\\
	=&\mathbb{E}_{\pi'}\left[\sum_{t=0}^{\infty}\gamma^t\Big[r(s_t,a_t)+\gamma \mathbb{E}_{s_{t+1}\sim P(\cdot|s_t,a_t)}[V^{\pi}_{\tau}(s_{t+1})|s_t,a_t]-V^{\pi}_{\tau}(s_t)\Big]
	\,\Big\vert\, s_0\sim\mu_0\right] \notag\\
	=&\mathbb{E}_{\pi'}\left[\sum_{t=0}^{\infty}\gamma^t\Big[Q^{\pi}(s_t,a_t)-V^{\pi}(s_t)\Big] \,\Big\vert\, s_0\sim\mu_0\right]\notag \\
	=&\frac{1}{1-\gamma}\mathbb{E}_{(s,a)\sim d^{\pi'}}\left[Q^{\pi}(s,a)-V^{\pi}(s)\rangle\right]\notag \\
	=&\frac{1}{1-\gamma}\mathbb{E}_{s\sim d^{\pi'}}\left[\langle Q^{\pi}(s,\cdot),\pi'(\cdot|s)-\pi(\cdot|s)\rangle\right],
	\label{eq:Vpiprime-Vpi-diff}
\end{align}
where the second to last step comes from the definition of $d^{\pi'}$ and the last step from the fact $V^{\pi}(s)=\mathbb{E}_{a\sim\pi(\cdot|s)}[Q^{\pi}(s,a)]$.

\section{Proof of Corollary~\ref{cor:sample unregularized}}
\label{proof:cor sample unregularized}
The proof consists of two steps. We first show that $J({\piz})-J({\pie})\leq\frac{\epsilon}{2}$ and then we bound $J({\pie})-J(\hat{\pi})$ by utilizing Theorem~\ref{thm:sample regularized}.

\paragraph{Step 1: Bounding $J({\piz})-J({\pie})$.}
Notice that ${\pie}$ is the solution to the regularized problem \eqref{eq:constrained}\eqref{eq:bellman flow 1}, therefore we have:
\begin{equation}
\mathbb{E}_{(s,a)\sim {\de}}[r(s,a)]-\alpha\mathbb{E}_{(s,a)\sim d^D}[f({\we}(s,a))]\geq\mathbb{E}_{(s,a)\sim {\dz}}[r(s,a)]-\alpha\mathbb{E}_{(s,a)\sim d^D}[f({\wz}(s,a))],
\end{equation}
which implies that
\begin{align}
J({\piz})-J({\pie})&=\mathbb{E}_{(s,a)\sim {\dz}}[r(s,a)]-\mathbb{E}_{(s,a)\sim {\de}}[r(s,a)]\\
&\leq\alpha\mathbb{E}_{(s,a)\sim d^D}[f({\wz}(s,a))]-\alpha\mathbb{E}_{(s,a)\sim d^D}[f({\we}(s,a))]\\
&\leq\alpha\mathbb{E}_{(s,a)\sim d^D}[f({\wz}(s,a))]\label{cor1-proof-eq1}\\
&\leq\alpha B^0_{f}\label{cor1-proof-eq2},
\end{align}
where (\ref{cor1-proof-eq1}) comes from the non-negativity of $f$ and (\ref{cor1-proof-eq2}) from the boundedness of $f$ when $\alpha=0$ (Assumption~\ref{ass:f prop}). Thus we have
\begin{equation}
\label{eq:cor1-proof-3}
J({\piz})-J({\pie})\leq\frac{\epsilon}{2}.
\end{equation}

\paragraph{Step 2: Bounding $J({\pie})-J(\hat{\pi})$.}
Using Theorem~\ref{thm:sample regularized}, we know that if
\begin{align}
	&n\geq\frac{131072\left(\epsilon {\BFE}+2{\BWE}{\BEE}{\BFZ}\right)^2}{\epsilon^6M_f^2(1-\gamma)^4}\cdot\log\frac{4|\VC||\WC|}{\delta}, \\
	&n_0\geq\frac{131072\left(2{\BVE}{\BFZ}\right)^2}{\epsilon^6M_f^2(1-\gamma)^2}\cdot\log\frac{4|\VC|}{\delta},
\end{align}
then with at least probability $1-\delta$, 
\begin{equation}
	\label{eq:cor1-proof-5}
	J({\pie})-J(\hat{\pi})\leq\frac{\epsilon}{2}.
\end{equation}


Using (\ref{eq:cor1-proof-3}) and (\ref{eq:cor1-proof-5}), we concludes that 
\begin{equation}
J({\piz})-J(\hat{\pi})\leq\epsilon
\end{equation} 
hold with at least probability $1-\delta$. This finishes our proof.

\section{Proof of Proposition~\ref{prop:LP stability}}
\label{proof:prop LP stability}
This proof largely follows \citet{mangasarian1979nonlinear}. First note that the regularized problem~\eqref{eq:constrained}\eqref{eq:bellman flow 1} has another more commonly used form of Lagrangian function:
\begin{equation}
\label{eq:Langrangian}
\bar{L}_{\alpha}(\lambda,\eta,w)=(1-\gamma)\mathbb{E}_{s\sim\mu_0}[\lambda(s)]-\alpha\mathbb{E}_{(s,a)\sim d^D}[f(w(s,a))]+\mathbb{E}_{(s,a)\sim d^D}[w(s,a)e_{\lambda}(s)]-\eta^{\top}w,
\end{equation} 
where $\lambda\in\mathbb{R}^{|\mathcal{S}|},\eta\in\mathbb{R}^{|\mathcal{S}||\mathcal{A}|}\geq0,w\in\mathbb{R}^{|\mathcal{S}||\mathcal{A}|}$. Let $(\lambda^*_{\alpha},\eta^*_{\alpha})=\arg\min_{\eta\geq0,\lambda\in\mathbb{R}^{|\mathcal{S}|}}\max_{w\in\mathbb{R}^{|\mathcal{S}||\mathcal{A}|}}\bar{L}_{\alpha}(\lambda,\eta,w)$, then we have the following lemma:
\begin{lemma}
\label{lem:Lagrangian equivalence}	
\begin{equation}
\lambda^*_{\alpha}={\va}.
\end{equation}	
\end{lemma}
\begin{proof}
The proof is referred to Appendix~\ref{proof:lem Lagrangian equivalence}.
\end{proof}
Due to Lemma~\ref{lem:Lagrangian equivalence}, we can only consider the primal optimum ${\wa}$ and the dual optimum $(\lambda^*_{\alpha},\eta^*_{\alpha})$ of the Lagrangian function (\ref{eq:Langrangian}). 

Let $\wzlp$ be the solution to the following optimization problem:
\begin{equation}
\max_{w\in\WCZ}-\alpha\mathbb{E}_{(s,a)\sim d^D}[f(w(s,a))]
\end{equation}

Then since $\wzlp\in\WCZ$, we know that $({{\wzlp}},\lambda^*_{0},\eta^*_{0})$ is the primal and dual optimum of the following constrained optimization problem, which is equivalent to the unregularized problem \eqref{prob:original problem}:
\begin{align}
&\max_{w}\sum_{s,a}[r(s,a)d^D(s,a)w(s,a)]\label{eq:constrained-alpha-0}\\
&\text{s.t. }\sum_{a}d^D(s,a)w(s,a)=(1-\gamma)\mu_0(s)+\gamma\sum_{s',a'}P(s|s',a')d^D(s',a')w(s',a')\label{eq:bellman-flow-alpha-0-1}\\
&\quad w(s,a)\geq0,\forall s,a\label{eq:bellman-flow-alpha-0-2}.
\end{align}
Let $p(s,a)$ denote $r(s,a)d^D(s,a)$ and $Aw=b$ denote the equality constraint (\ref{eq:bellman-flow-alpha-0-1}), then we can obtain the following LP:
\begin{align}
&\min_{w} -p^{\top}w\label{eq:LP-0}\\
&\text{s.t. }Aw=b\label{eq:LP-0-constraint-1}\\
&\quad w(s,a)\geq0,\forall s,a\label{eq:LP-0-constraint-2}.
\end{align}
By the KKT conditions of the above problem, we can obtain:
\begin{align}
&A^{\top}\lambda^*_{0}-p-\eta^*_{0}=0,\\
&A{{\wzlp}}=b, {{\wzlp}}\geq0,\\
&\eta^*_{0}\geq0,\\
&\eta^*_{0}(s,a){{\wzlp}}(s,a)=0,\forall s,a.\\
\end{align}

Let $c=-p^{\top}{\wzlp}$. Next we construct an auxiliary constrained optimization problem:
\begin{align}
&\min_{w} \mathbb{E}_{(s,a)\sim d^D}[f(w(s,a))]\label{eq:auxiliary}\\
&\text{s.t. }Aw=b, \label{eq:lp1}\\
&\quad w(s,a)\geq0,\forall s,a, \label{eq:lp2}\\
&-p^{\top}w\leq c \label{eq:lp3}.
\end{align}
Then the corresponding Lagrangian function is
\begin{equation}
\mathbb{E}_{(s,a)\sim d^D}[f(w(s,a))]+\lambda_{aux}^{\top}(Aw-b)-\eta_{aux}^{\top}w+\xi_{aux}(-p^{\top}w-c).
\end{equation}
Denote the primal and dual optimum of the auxiliary problem by $({w}^*_{aux},\lambda^*_{aux},\eta^*_{aux},\xi^*_{aux})$. Then obviously the constraints \eqref{eq:lp1}\eqref{eq:lp2}\eqref{eq:lp3} are equivalent to $w\in\WCZ$ and therefore ${w}^*_{aux}=\wzlp$, implying that $(\wzlp,\lambda^*_{aux},\eta^*_{aux},\xi^*_{aux})$ satisfies the following KKT conditions:
\begin{align}
&d^D\circ\nabla f({\wzlp})+A^{\top}\lambda^*_{aux}-\eta^*_{aux}-\xi^*_{aux}p=0,\\
&A{\wzlp}=b, {\wzlp}\geq0, -p^{\top}{\wzlp}=c,\\
&\eta^*_{aux}\geq0,\xi^*_{aux}\geq0,\\
&\eta^*_{aux}(s,a){\wzlp}(s,a)=0,\forall s,a,
\end{align}
where $d^D\circ\nabla f({\wzlp})$ denotes product by element.

Now we look at KKT conditions of (\ref{eq:Langrangian}):
\begin{align}
&A^{\top}\lambda^*_{\alpha}-p-\eta^*_{\alpha}+\alpha d^D\circ\nabla f({\wa})=0,\\
&A{\wa}=b, {\wa}\geq0,\\
&\eta^*_{\alpha}\geq0,\\
&\eta^*_{\alpha}(s,a){\wa}(s,a)=0,\forall s,a.\\
\end{align}

\begin{itemize}
\item \textbf{When $\mathbf{\xi^*_{aux}=0}$.} It can be easily checked that $({\wa}={\wzlp},\lambda^*_{\alpha}=\lambda^*_{0}+\alpha\lambda^*_{aux},\eta^*_{\alpha}=\eta^*_{0}+\alpha\eta^*_{aux})$ satisfies the KKT conditions of (\ref{eq:Langrangian}) for all $\alpha\geq0$.

\item \textbf{When $\mathbf{\xi^*_{aux}>0}$.} It can be easily checked that $({\wa}={\wzlp},\lambda^*_{\alpha}=(1-\alpha\xi^*_{aux})\lambda^*_{0}+\alpha\lambda^*_{aux},\eta^*_{\alpha}=(1-\alpha\xi^*_{aux})\eta^*_{0}+\alpha\eta^*_{aux})$ satisfies the KKT conditions of (\ref{eq:Langrangian}) for $\alpha\in[0,\bar{\alpha}]$ where $\bar{\alpha}=\frac{1}{\xi^*_{aux}}$.
\end{itemize}
Therefore, when $\alpha\in[0,\bar{\alpha}]$, $({\wa}={\wzlp},\lambda^*_{\alpha}=(1-\alpha\xi^*_{aux})\lambda^*_{0}+\alpha\lambda^*_{aux},\eta^*_{\alpha}=(1-\alpha\xi^*_{aux})\eta^*_{0}+\alpha\eta^*_{aux})$ is the primal and dual optimum of (\ref{eq:Langrangian}). Then by Lemma~\ref{lem:Lagrangian equivalence}, we know for $\alpha\in[0,\bar{\alpha}]$,
\begin{equation}
{\wa}={\wzlp}\in W^*_0, {\va}=(1-\alpha\xi^*_{aux})\lambda^*_{0}+\alpha\lambda^*_{aux}.
\end{equation}
Let $\alpha=\bar{\alpha}=\frac{1}{\xi^*_{aux}}$, then since $\Vert\wab\Vert_{\infty}=\Vert {\wzlp}\Vert_{\infty}\leq B^0_w$, by Lemma~\ref{lem:bound tilde nu} we have:
\begin{equation}
\Vert\bar{\alpha}\lambda^*_{aux}\Vert_{\infty}=\Vert\vab\Vert_{\infty}\leq\frac{\bar{\alpha}{\BFPZ}+1}{1-\gamma},
\end{equation}
which implies that
\begin{equation}
\Vert\lambda^*_{aux}\Vert_{\infty}\leq\frac{{\BFPZ}+\xi^*_{aux}}{1-\gamma}.
\end{equation}
Therefore, combining with $\Vert{\vz}\Vert_{\infty}\leq\frac{1}{1-\gamma}$, we have
\begin{equation}
\Vert{\va}-{\vz}\Vert_{\infty}\leq\alpha\cdot\frac{{\BFPZ}+2\xi^*_{aux}}{1-\gamma},\forall \alpha\in[0,\bar{\alpha}]
\end{equation}
which concludes our proof.

\subsection{Proof of Lemma~\ref{lem:Lagrangian equivalence}}
\label{proof:lem Lagrangian equivalence}
From KKT conditions of $\bar{L}_{\alpha}(\lambda,\eta,w)$, we have
\begin{align}
&{\wa}(s,a)=(f')^{-1}(\frac{e_{\lambda^*_{\alpha}}(s,a)+\eta^*_{\alpha}(s,a)}{\alpha}),\forall s,a,\\
&{\wa}\geq0,\\
&\sum_{a}{\wa}(s,a)d^D(s,a)=(1-\gamma)\mu_0(s)+\gamma\sum_{s',a'}P(s|s',a'){\wa}(s',a')d^D(s',a'),\forall s,\\
&\eta^*_{\alpha}\geq0,\\
&\eta^*_{\alpha}(s,a){\wa}(s,a)=0,\forall s,a.
\end{align}
Therefore, we can see that $\lambda^*_{\alpha}$ is the solution of the following equations:
\begin{align}
&e_{\lambda^*_{\alpha}}(s,a)=\alpha f'({\wa}(s,a)),\text{ for }s,a \text{ such that }{\wa}(s,a)\neq0,\label{eq:lagrange-1}\\
&e_{\lambda^*_{\alpha}}(s,a)\leq\alpha f'(0),\text{ for }s,a \text{ such that }{\wa}(s,a)=0.\label{eq:lagrange-2}
\end{align}

Besides, from KKT conditions of ${\La}({v},w)$, we have
\begin{align}
&{\wa}(s,a)=\max\{0,(f')^{-1}(\frac{e_{\lambda^*_{\alpha}}(s,a)}{\alpha})\},\forall s,a,\\
&{\wa}\geq0,\\
&\sum_{a}{\wa}(s,a)d^D(s,a)=(1-\gamma)\mu_0(s)+\gamma\sum_{s',a'}P(s|s',a'){\wa}(s',a')d^D(s',a'),\forall s.
\end{align}
Therefore,  ${\va}$ is the solution of the following equations:
\begin{align}
&e_{{\va}}(s,a)=\alpha f'({\wa}(s,a)),\text{ for }s,a \text{ such that }{\wa}(s,a)\neq0,\label{eq:lagrange-3}\\
&e_{{\va}}(s,a)\leq\alpha f'(0),\text{ for }s,a \text{ such that }{\wa}(s,a)=0.\label{eq:lagrange-4}
\end{align}

It is observed that (\ref{eq:lagrange-1})(\ref{eq:lagrange-2}) is the same as (\ref{eq:lagrange-3})(\ref{eq:lagrange-4}), which implies that $\lambda^*_{\alpha}={\va}$.

\section{Proof of Theorem~\ref{thm:sample regularized approximate}}
\label{proof:thm sample regularized approximate}
Our proof follows a similar procedure of Theorem~\ref{thm:sample regularized} and also consists of (1) bounding $|{\La}({v},w)-\hat{L}_{\alpha}({v},w)|$, (2) characterizing the error $\Vert\hat{w}-{\wa}\Vert_{2,d^D}$ and (3) analyzing $\hat{\pi}$ and ${\pia}$. The first and third step are exactly the same as Theorem~\ref{thm:sample regularized} but the second step will be more complicated, on which we will elaborate on in this section. We will use the following notations for brevity throughout the discussion:
\begin{align}
&{\vav}=\arg\min_{{v}\in {\VC}}\Vert {v}-{\va}\Vert_{1,\mu_0}+\Vert {v}-{\va}\Vert_{1,d^D}+\Vert {v}-{\va}\Vert_{1,d^{D'}},\\
&{\waww}=\arg\min_{w\in {\WC}}\Vert w-{\wa}\Vert_{1,d^D},\\
&\hat{{v}}(w)=\arg\min_{{v}\in {\VC}}\hat{L}_{\alpha}({v},w),\forall w.
\end{align}

We first need to characterize ${\La}({\va},{\wa})-{\La}({\va},\hat{w})$. Similarly, we decompose ${\La}({\va},{\wa})-{\La}({\va},\hat{w})$ into the following terms:
\begin{align}
&{\La}({\va},\hat{w})-{\La}({\va},{\wa})=(\underbrace{{\La}({\va},\hat{w})-{\La}({\vav},\hat{w})}_{(1)})+(\underbrace{{\La}({\vav},\hat{w})-\hat{L}_{\alpha}({\vav},\hat{w})}_{(2)})\\ &+(\underbrace{\hat{L}_{\alpha}({\vav},\hat{w})-\hat{L}_{\alpha}(\hat{{v}},\hat{w})}_{(3)})+(\underbrace{\hat{L}_{\alpha}(\hat{{v}},\hat{w})-\hat{L}_{\alpha}(\hat{{v}}({\waww}),{\waww})}_{(4)})\\ &+(\underbrace{\hat{L}_{\alpha}(\hat{{v}}({\waww}),{\waww})-{\La}(\hat{{v}}({\waww}),{\waww})}_{(5)})+ (\underbrace{{\La}(\hat{{v}}({\waww}),{\waww})-{\La}(\hat{{v}}({\waww}),{\wa})}_{(6)}),\\
&+ (\underbrace{{\La}(\hat{{v}}({\waww}),{\wa})-{\La}({\va},{\wa})}_{(7)}).
\end{align}

For term (2) and (5), we can apply Lemma~\ref{lem:hat L conc} and thus
\begin{equation}
(2)\geq-{\lstat},(5)\geq-{\lstat}.
\end{equation}
For term (3), since $\hat{L}_{\alpha}(\hat{{v}},\hat{w})-\min_{{v}\in {\VC}}\hat{L}_{\alpha}({v},\hat{w})\leq\epsilon_{o,{v}}$ and ${\vav}\in {\VC}$, we have
\begin{equation}
(3)\geq-\epsilon_{o,{v}}.
\end{equation}
For term (4), since $\max_{w\in {\WC}}\min_{{v}\in {\VC}}\hat{L}_{\alpha}({v},w)-\min_{{v}\in {\VC}}\hat{L}_{\alpha}({v},\hat{w})\leq\epsilon_{o,w}$ and ${\waww}\in {\WC}$, 
\begin{equation}
\hat{L}_{\alpha}(\hat{{v}},\hat{w})\geq\min_{{v}\in {\VC}}\hat{L}_{\alpha}({v},\hat{w})\geq\max_{w\in {\WC}}\min_{{v}\in {\VC}}\hat{L}_{\alpha}({v},w)-\epsilon_{o,w}\geq\hat{L}_{\alpha}(\hat{{v}}({\waww}),{\waww})-\epsilon_{o,w},
\end{equation}
or
\begin{equation}
(4)\geq-\epsilon_{o,w}.
\end{equation}

For term (7), since ${\va}=\arg\min_{{v}\in\mathbb{R}^{|\mathcal{S}|}}{\La}({v},{\wa})$, we have:
\begin{equation}
(7)\geq0.
\end{equation}

There are only term (1) and (6) left to be bounded, for which we introduce the following lemma on the continuity of ${\La}({v},w)$,
\begin{lemma}
	\label{lem:f continuity}
	Suppose Assumption~\ref{ass:W bound},\ref{ass:f prop},\ref{ass:V bound} hold. Then for any ${v},{v}_1,{v}_2\in {\VC}$ and $w,w_1,w_2\in {\WC}$, we have:
	\begin{align}
	&|{\La}({v}_1,w)-{\La}({v}_2,w)|\leq\left({\BW}+1\right)\left(\Vert{v}_1-{v}_2\Vert_{1,\mu_0}+\Vert{v}_1-{v}_2\Vert_{1,d^D}+\Vert{v}_1-{v}_2\Vert_{1,d^{D'}}\right),\\
	&|{\La}({v},w_1)-{\La}({v},w_2)|\leq({\BE}+\alpha {\BFP})\Vert w_1-w_2\Vert_{1,d^D}.
	\end{align}
\end{lemma}
The proof is in Section~\ref{proof:lem f continuity}. Using Lemma~\ref{lem:f continuity}, we can bound term (1) and (6) easily:
\begin{equation}
(1)\geq-\left({\BW}+1\right)\erva, (6)-\geq({\BE}+\alpha {\BFP})\erwa.
\end{equation}

Combining the above inequalities, it is obvious that
\begin{equation}
{\La}({\va},\hat{w})-{\La}({\va},{\wa})\geq-2{\lstat}-(\epsilon_{o,{v}}+\epsilon_{o,w})-\left(\left({\BW}+1\right)\erva+({\BE}+\alpha {\BFP})\erwa\right).
\end{equation}

Let $\eag$ denote $\left({\BW}+1\right)\erva+({\BE}+\alpha {\BFP})\erwa$ and $\epsilon_{opt}$ denote $\epsilon_{o,{v}}+\epsilon_{o,w}$, then
\begin{equation}
{\La}({\va},\hat{w})-{\La}({\va},{\wa})\geq-2{\lstat}-\epsilon_{opt}-\eag.
\end{equation}
Further we utilize the strong convexity of $f$ and Lemma~\ref{lem:tilde w tilde pi}, then we have:
\begin{align}
&\mathbb{E}_{s\sim {\da}}[\Vert{\pia}(\cdot|s)-\hat{\pi}(\cdot|s)\Vert_1]\leq 2\Vert \hat{w}-{\wa}\Vert_{2,d^D}\leq4\sqrt{\frac{\lstat}{\alpha M_f}}+2\sqrt{\frac{2(\epsilon_{opt}+\eag)}{\alpha M_f}},
\end{align}
which completes the proof.

\subsection{Proof of Lemma~\ref{lem:f continuity}}
\label{proof:lem f continuity}
First, by the definition of ${\La}({v},w)$ (\ref{prob:maximin2}) we have
\begin{align}
&|{\La}({v}_1,w)-{\La}({v}_2,w)|\\
=&|(1-\gamma)\mathbb{E}_{s\sim \mu_0}[{v}_1(s)-{v}_2(s)]+\mathbb{E}_{(s,a)\sim d^D}[w(s,a)(e_{{v}_1}(s,a)-e_{{v}_2}(s,a))]|\\
\leq&(1-\gamma)\mathbb{E}_{s\sim \mu_0}[|{v}_1(s)-{v}_2(s)|]+\mathbb{E}_{(s,a)\sim d^D}[w(s,a)|e_{{v}_1}(s,a)-e_{{v}_2}(s,a)|]\\
=&(1-\gamma)\Vert{v}_1-{v}_2\Vert_{1,\mu_0}+\mathbb{E}_{(s,a)\sim d^D}[w(s,a)|e_{{v}_1}(s,a)-e_{{v}_2}(s,a)|].
\end{align}

For $\mathbb{E}_{(s,a)\sim d^D}[w(s,a)|e_{{v}_1}(s,a)-e_{{v}_2}(s,a)|]$, notice that from Assumption~\ref{ass:W bound},
\begin{align}
&\mathbb{E}_{(s,a)\sim d^D}[w(s,a)|e_{{v}_1}(s,a)-e_{{v}_2}(s,a)|]\\
\leq&{\BW}\mathbb{E}_{(s,a)\sim d^D}\left[|\gamma\mathbb{E}_{s'\sim P(s'|s,a)}[{v}_1(s')-{v}_2(s')]+\left({v}_2(s)-{v}_1(s)\right)|\right]\\
\leq&{\BW}\mathbb{E}_{(s,a)\sim d^D}\left[|\gamma\mathbb{E}_{s'\sim P(s'|s,a)}[{v}_1(s')-{v}_2(s')]|\right]+{\BW}\mathbb{E}_{s\sim d^D}\left[|{v}_2(s)-{v}_1(s)|\right]\\
\leq&\gamma {\BW}\mathbb{E}_{(s,a)\sim d^D,s'\sim P(s'|s,a)}[|{v}_1(s')-{v}_2(s')|]+{\BW}\Vert{v}_2-{v}_1\Vert_{1,d^D}\\
\leq& {\BW}\left(\Vert{v}_1-{v}_2\Vert_{1,d^D}+\Vert{v}_1-{v}_2\Vert_{1,d^{D'}}\right).
\end{align}

Thus we have
\begin{equation}
|{\La}({v}_1,w)-{\La}({v}_2,w)|\leq\left({\BW}+1\right)\left(\Vert{v}_1-{v}_2\Vert_{1,\mu_0}+\Vert{v}_1-{v}_2\Vert_{1,d^D}+\Vert{v}_1-{v}_2\Vert_{1,d^{D'}}\right).
\end{equation}

Next we bound $|{\La}({v},w_1)-{\La}({v},w_2)|$:
\begin{align}
&|{\La}({v},w_1)-{\La}({v},w_2)|\\
=&|\alpha\mathbb{E}_{(s,a)\sim d^D}[f(w_2(s,a))-f(w_1(s,a))]+\mathbb{E}_{(s,a)\sim d^D}[(w_1(s,a)-w_2(s,a))e_{{v}}(s,a)]|\\
\leq&\alpha\mathbb{E}_{(s,a)\sim d^D}[|f(w_1(s,a))-f(w_2(s,a))|]+\mathbb{E}_{(s,a)\sim d^D}[|w_1(s,a)-w_2(s,a)|e_{{v}}(s,a)].
\end{align}

For $\alpha\mathbb{E}_{(s,a)\sim d^D}[|f(w_1(s,a))-f(w_2(s,a))|]$, from Assumption~\ref{ass:f prop} we know
\begin{align}
&\alpha\mathbb{E}_{(s,a)\sim d^D}[|f(w_1(s,a))-f(w_2(s,a))|]\\
\leq&\alpha {\BFP}\mathbb{E}_{(s,a)\sim d^D}[|w_1(s,a)-w_2(s,a)|]\\
=&\alpha {\BFP}\Vert w_1-w_2\Vert_{1,d^D}.
\end{align}

For $\mathbb{E}_{(s,a)\sim d^D}[|w_1(s,a)-w_2(s,a)|e_{{v}}(s,a)]$, from Assumption~\ref{ass:V bound} we know
\begin{align}
&\mathbb{E}_{(s,a)\sim d^D}[|w_1(s,a)-w_2(s,a)|e_{{v}}(s,a)]\\
\leq&{\BE}\mathbb{E}_{(s,a)\sim d^D}[|w_1(s,a)-w_2(s,a)|]\\
=&{\BE}\Vert w_1-w_2\Vert_{1,d^D}.
\end{align}

Therefore we have
\begin{equation}
|{\La}({v},w_1)-{\La}({v},w_2)|\leq({\BE}+\alpha {\BFP})\Vert w_1-w_2\Vert_{1,d^D}.
\end{equation}

\section{Proof of Lemmas in Theorem~\ref{thm:sample regularized constrained}}
\subsection{Proof of Lemma~\ref{lem:bound tilde nu 2}}
\label{proof:lem bound tilde nu 2}
From KKT conditions of the maximin problem (\ref{prob:maximin2 constrained}), we have  
\begin{equation}
{\waw}(s,a)=\min\left(\max\left(0,(f')^{-1}\left(\frac{e_{{\vaw}}(s,a)}{\alpha}\right)\right),B_w\right).
\end{equation}
Suppose $|{\vaw}(s_m)|=\Vert {\vaw}\Vert_{\infty}$. Then we can consider the following two cases separately. 

\begin{itemize}
\item \textbf{If there exists $a_{s_m}\in\mathcal{A}$ such that $0<{\waw}(s_m,a_{s_m})<B_w$.} 

In this case, we know that 
\begin{equation}
|e_{{\vaw}}(s_m,a_{s_m})|=\alpha |f'({\waw}(s_m,a_{s_m}))|\leq\alpha B_{f'}.
\end{equation}
Then we can follow the arguments in Appendix~\ref{proof:lemma bound tilde nu} to obtain:
\begin{equation}
\Vert{\vaw}\Vert_{\infty}\leq\frac{\alpha B_{f'}+1}{1-\gamma}.
\end{equation}

\item \textbf{If for all $a\in\mathcal{A}$, ${\waw}(s_m,a)\in\{0,B_w\}$.}
In this case, we first introduce the following lemma:
\begin{lemma}
\label{lem:bound nu bridge}
If for all $a\in\mathcal{A}$, ${\waw}(s_m,a)\in\{0,B_w\}$, then there exist $a_1,a_2\in\mathcal{A}$ such that ${\waw}(s_m,a_1)=0,{\waw}(s_m,a_2)=B_w$.
\end{lemma}
See Appendix~\ref{proof:lem bound nu bridge} for proof. With Lemma~\ref{lem:bound nu bridge}, we can bound $|{\vaw}(s_m)|$ as follows.

If ${\vaw}(s_m)\geq0$, then since ${\waw}(s_m,a_2)=B_w$, we know $e_{{\vaw}}(s_m,a_2)\geq \alpha f'(B_w)$. Therefore we have:
\begin{equation}
\alpha f'(B_w)\leq e_{{\vaw}}(s_m,a_2)\leq r(s_m,a_2)-(1-\gamma){\vaw}(s_m),
\end{equation}
which implies:
\begin{equation}
\label{eq:bound nu 1}
{\vaw}(s_m)\leq\frac{1}{1-\gamma}|r(s_m,a_2)+\alpha f'(B_w)|\leq\frac{\alpha B_{f'}+1}{1-\gamma}.
\end{equation}

If ${\vaw}(s_m)<0$, then since ${\waw}(s_m,a_1)=0$, we know $e_{{\vaw}}(s_m,a_1)\leq \alpha f'(0)$. Therefore we have:
\begin{equation}
\alpha f'(0)\geq e_{{\vaw}}(s_m,a_1)\geq r(s_m,a_2)-(1-\gamma){\vaw}(s_m),
\end{equation}
which implies:
\begin{equation}
\label{eq:bound nu 2}
{\vaw}(s_m)\geq-\frac{1}{1-\gamma}(|r(s_m,a_1)|+|\alpha f'(0)|)\geq-\frac{\alpha B_{f'}+1}{1-\gamma}.
\end{equation}

Combining (\ref{eq:bound nu 1}) and (\ref{eq:bound nu 2}), we have $\Vert{\vaw}\Vert_{\infty}=|{\vaw}(s_m)|\leq\frac{\alpha B_{f'}+1}{1-\gamma}$.
\end{itemize}

In conclusion, we have:
\begin{equation}
\Vert{\vaw}\Vert_{\infty}\leq\frac{\alpha B_{f'}+1}{1-\gamma}.
\end{equation}

\subsection{Proof of Lemma~\ref{lem:bound nu bridge}}
\label{proof:lem bound nu bridge}
First note that it is impossible to have ${\waw}(s_m,a)=0,\forall a$. This is because ${\daw}(s_m,a)={\waw}(s_m,a)d^D(s_m,a)$ satisfies Bellman flow constraint (\ref{eq:bellman flow 1}). Therefore
\begin{equation}
{\daw}(s_m)=\sum_a{\waw}(s_m,a)d^D(s_m,a)\geq(1-\gamma)\mu_0(s_m)>0.
\end{equation}

On the other hand, if ${\waw}(s_m,a)=B_w,\forall a$, then from Bellman flow constraints we have:
\begin{equation}
\label{eq:bound nu 3}
B_wd^D(s_m)={\daw}(s_m)=(1-\gamma)\mu_0(s_m)+\sum_{s',a'}P(s_m|s',a'){\waw}(s',a')d^D(s',a').
\end{equation}
Notice that from Assumption~\ref{ass:dataset} $d^D$ is the discounted visitation distribution of $\pi_D$ and thus also satisfies Bellman flow constraints:
\begin{equation}
d^D(s_m)=(1-\gamma)\mu_0(s_m)+\sum_{s',a'}P(s_m|s',a')d^D(s',a'),
\end{equation}
which implies
\begin{equation}
\label{eq:bound nu 4}
B_wd^D(s_m)=(1-\gamma)B_w\mu_0(s_m)+\sum_{s',a'}B_wP(s_m|s',a')d^D(s',a').
\end{equation}

Combining (\ref{eq:bound nu 3}) and (\ref{eq:bound nu 4}), we have
\begin{equation}
(1-\gamma)(B_w-1)\mu_0(s_m)=\sum_{s',a'}({\waw}-B_w)P(s_m|s',a')d^D(s',a').
\end{equation}
However, since $B_w>1,\mu_0(s_m)>0,{\waw}-B_w\leq0$, we have:
\begin{equation}
(1-\gamma)(B_w-1)\mu_0(s_m)>0,\sum_{s',a'}({\waw}-B_w)P(s_m|s',a')d^D(s',a')\leq0,
\end{equation}
which is a contradiction. 

Therefore,  there must exist $a_1,a_2\in\mathcal{A}$ such that ${\waw}(s_m,a_1)=0,{\waw}(s_m,a_2)=B_w$.
\section{Proof of \algbc}
\subsection{Proof of Lemma~\ref{lem:variation}}
\label{proof:lem variation}
Notice that by the variational form of total variation, we have for any policies $\pi,\pi'$ and $s\in\mathcal{S}$,
\begin{align}
\Vert\pi(\cdot|s)-\pi'(\cdot|s)\Vert_1&=\max_{h:\Vert h\Vert_{\infty}\leq 1}[\mathbb{E}_{a\sim\pi(\cdot|s)} h(a)-\mathbb{E}_{a\sim\pi'(\cdot|s)} h(a)]\\
&=\mathbb{E}_{a\sim\pi(\cdot|s)}[ h^s_{\pi,\pi'}(a)]-\mathbb{E}_{a\sim\pi'(\cdot|s)}[h^s_{\pi,\pi'}(a)],
\end{align}
which implies that
\begin{align}
\mathbb{E}_{s\sim d}[\Vert\pi(\cdot|s)-\pi'(\cdot|s)\Vert_1]&=\mathbb{E}_{s\sim d}\left[\mathbb{E}_{a\sim\pi(\cdot|s)}[h^s_{\pi,\pi'}(a)]-\mathbb{E}_{a\sim\pi'(\cdot|s)}[h^s_{\pi,\pi'}(a)]\right]\\
&=\mathbb{E}_{s\sim d}\left[\mathbb{E}_{a\sim\pi(\cdot|s)}[h_{\pi,\pi'}(s,a)]-\mathbb{E}_{a\sim\pi'(\cdot|s)}[h_{\pi,\pi'}(s,a)]\right],
\end{align}
where the last step comes from the definition of $h_{\pi,\pi'}$.

\subsection{Proof of Theorem~\ref{thm:sample regularized BC}}
\label{proof:thm sample regularized BC}
Let $\epsilon_{UO}$ denote $\left(\frac{4(\alpha {\BF}+{\BW}{\BE})}{\alpha M_f}\right)^{\frac{1}{2}}\cdot\left(\frac{2\log\frac{8|{\VC}||{\WC}|}{\delta}}{n_1}\right)^{\frac{1}{4}}+\left(\frac{4(1-\gamma)\BV}{\alpha M_f}\right)^{\frac{1}{2}}\cdot\left(\frac{2\log\frac{8|{\VC}|}{\delta}}{n_0}\right)^{\frac{1}{4}}$. Suppose $E$ denote the event
\begin{equation}
\Vert \hat{w}-{\wa}\Vert_{2,d^D}\leq\epsilon_{UO},
\end{equation} 
then by Theorem~\ref{thm:sample regularized approximate}, we have
\begin{equation}
\text{Pr}(E)\geq 1-\frac{\delta}{2}.
\end{equation}

Our following discussion is all conditioned on $E$. Let $l'_{i,\pi,h}$ denote $\hat{w}(s_i,a_i)(h^{\pi}(s_i)-h(s_i,a_i))$ then we know:
\begin{align}
\mathbb{E}_{\mathcal{D}_2}[l'_{i,\pi,h}]&=\mathbb{E}_{(s,a)\sim d^D}[\hat{w}(s,a)(h^{\pi}(s)-h(s,a))]\\
&=\left(\sum_{s,a}d^D(s,a)\hat{w}(s,a)\right)\mathbb{E}_{s\sim\hat{d}'}\left[\mathbb{E}_{a\sim\pi(\cdot|s)}[h(s,a)]-\mathbb{E}_{a\sim\hat{\pi}(\cdot|s)}[h(s,a)]\right],
\end{align}
where $\hat{d}'(s)=\frac{\sum_{a'}d^D(s,a')\hat{w}(s,a')}{\sum_{s',a'}d^D(s',a')\hat{w}(s',a')}$. Notice that $0\leq\hat{w}(s,a)\leq {\BW},|h(s,a)|\leq1$, then by Hoeffding's inequality we have for any $\pi\in\Pi$ and $h\in\mathcal{H}$, with at least probability $1-\frac{\delta}{2}$,
\begin{align}
&\bigg|\frac{1}{n_2}\sum_{i=1}^{n_2}l'_{i,\pi,h}-\left(\sum_{s',a'}d^D(s',a')\hat{w}(s',a')\right)\mathbb{E}_{s\sim\hat{d}'}\left[\mathbb{E}_{a\sim\pi(\cdot|s)}[h(s,a)]-\mathbb{E}_{a\sim\hat{\pi}(\cdot|s)}[h(s,a)]\right]\bigg|\notag\\
\leq&2{\BW}\sqrt{\frac{2\log\frac{4|\mathcal{H}||\Pi|}{\delta}}{n_2}}\leq2{\BW}\sqrt{\frac{6\log\frac{4|\Pi|}{\delta}}{n_2}}:={\lstatt}.\label{eq:thm3-4}
\end{align}

Besides, the following lemma shows that $\hat{d}'$ is close to ${\da}$ and $\left(\sum_{s',a'}d^D(s',a')\hat{w}(s',a')\right)$ is close to 1 conditioned on $E$:
\begin{lemma}
\label{lem:importance sampling close}
Conditioned on $E$, we have
\begin{align}
\Vert\hat{d}'-{\da}\Vert_1\leq 2\epsilon_{UO},\label{eq:thm3-1}\\
\bigg|\left(\sum_{s',a'}d^D(s',a')\hat{w}(s',a')\right)-1\bigg|\leq\epsilon_{UO}.\label{eq:thm3-2}
\end{align}
\end{lemma}

The proof of the above lemma is in Appendix~\ref{proof:lem importance sampling close}.

With concentration result (\ref{eq:thm3-4}) and Lemma~\ref{lem:importance sampling close}, we can bound $\mathbb{E}_{s\sim{\da}}[\Vert\bar{\pi}(\cdot|s)-{\pia}(\cdot|s)\Vert_1]$. To facilitate our discussion, we will use the following notations:
\begin{align}
&\bar{h}:=h_{\bar{\pi},{\pia}}\in\mathcal{H},\\
&\bar{h}':=\arg\max_{h\in\mathcal{H}}\sum_{i=1}^{n_2}\hat{w}(s_i,a_i)[h^{\bar{\pi}}(s_i)-h(s_i,a_i)],\\
&\tilde{h}:=\arg\max_{h\in\mathcal{H}}\sum_{i=1}^{n_2}\hat{w}(s_i,a_i)[h^{{\pia}}(s_i)-h(s_i,a_i)].
\end{align}

Then we have
\begin{align}
&\mathbb{E}_{s\sim{\da}}[\Vert\bar{\pi}(\cdot|s)-{\pia}(\cdot|s)\Vert_1]\\
\leq&\mathbb{E}_{s\sim\hat{d}'}[\Vert\bar{\pi}(\cdot|s)-{\pia}(\cdot|s)\Vert_1]+4\epsilon_{UO}\\
=&\mathbb{E}_{s\sim\hat{d}'}[\mathbb{E}_{a\sim\bar{\pi}(\cdot|s)}[\bar{h}(s,a)]-\mathbb{E}_{a\sim{\pia}(\cdot|s)}[\bar{h}(s,a)]]+4\epsilon_{UO}\\
=&\mathbb{E}_{s\sim\hat{d}'}[\mathbb{E}_{a\sim\bar{\pi}(\cdot|s)}[\bar{h}(s,a)]-\mathbb{E}_{a\sim\hat{\pi}(\cdot|s)}[\bar{h}(s,a)]]\notag\\
&+\mathbb{E}_{s\sim\hat{d}'}[\mathbb{E}_{a\sim\hat{\pi}(\cdot|s)}[\bar{h}(s,a)]-\mathbb{E}_{a\sim{\pia}(\cdot|s)}[\bar{h}(s,a)]]+4\epsilon_{UO}\\
=&\mathbb{E}_{s\sim\hat{d}',a\sim\hat{\pi}(\cdot|s)}[\bar{h}^{\bar{\pi}}(s)-\bar{h}(s,a)]+\mathbb{E}_{s\sim\hat{d}',a\sim\hat{\pi}(\cdot|s)}[(-\bar{h}^{{\pia}}(s))-(-\bar{h}(s,a))]+4\epsilon_{UO}\\
\leq&\mathbb{E}_{s\sim\hat{d}',a\sim\hat{\pi}(\cdot|s)}[\bar{h}^{\bar{\pi}}(s)-\bar{h}(s,a)]+\mathbb{E}_{s\sim\hat{d}'}[\Vert{\pia}(\cdot|s)-\hat{\pi}(\cdot|s)\Vert_1]+4\epsilon_{UO}\\
\leq&\mathbb{E}_{s\sim\hat{d}',a\sim\hat{\pi}(\cdot|s)}[\bar{h}^{\bar{\pi}}(s)-\bar{h}(s,a)]+\mathbb{E}_{s\sim{\da}}[\Vert{\pia}(\cdot|s)-\hat{\pi}(\cdot|s)\Vert_1]+8\epsilon_{UO}\\
\leq&\mathbb{E}_{s\sim\hat{d}',a\sim\hat{\pi}(\cdot|s)}[\bar{h}^{\bar{\pi}}(s)-\bar{h}(s,a)]+10\epsilon_{UO},\label{eq:thm3-5}
\end{align}
where the first and sixth steps come from (\ref{eq:thm3-1}), the fifth step is due to $\Vert\bar{h}\Vert_{\infty}\leq1$ and the last step from Theorem~\ref{thm:sample regularized approximate}.

For $\mathbb{E}_{s\sim\hat{d}',a\sim\hat{\pi}(\cdot|s)}[\bar{h}^{\bar{\pi}}(s)-\bar{h}(s,a)]$, we utilize concentration result (\ref{eq:thm3-4}) and have with at least probability $1-\delta$:
\begin{align}
&\mathbb{E}_{s\sim\hat{d}',a\sim\hat{\pi}(\cdot|s)}[\bar{h}^{\bar{\pi}}(s)-\bar{h}(s,a)]\\
\leq&\left(\sum_{s',a'}d^D(s',a')\hat{w}(s',a')\right)\mathbb{E}_{s\sim\hat{d}'}\left[\mathbb{E}_{a\sim\bar{\pi}(\cdot|s)}[\bar{h}(s,a)]-\mathbb{E}_{a\sim\hat{\pi}(\cdot|s)}[\bar{h}(s,a)]\right]+2\epsilon_{UO}\\
\leq&\frac{1}{n_2}\sum_{i=1}^{n_2}[\hat{w}(s_i,a_i)(\bar{h}^{\bar{\pi}}(s_i)-\bar{h}(s_i,a_i))]+{\lstatt}+2\epsilon_{UO}\\
\leq&\frac{1}{n_2}\sum_{i=1}^{n_2}[\hat{w}(s_i,a_i)(\bar{h}'^{\bar{\pi}}(s_i)-\bar{h}'(s_i,a_i))]+{\lstatt}+2\epsilon_{UO}\\
\leq&\frac{1}{n_2}\sum_{i=1}^{n_2}[\hat{w}(s_i,a_i)(\tilde{h}^{{\pia}}(s_i)-\tilde{h}(s_i,a_i))]+{\lstatt}+2\epsilon_{UO}\\
\leq&\mathbb{E}_{s\sim\hat{d}',a\sim\hat{\pi}(\cdot|s)}[\tilde{h}^{{\pia}}(s)-\tilde{h}(s,a)]+2{\lstatt}+4\epsilon_{UO}\\
\leq&\mathbb{E}_{s\sim\hat{d}'}[\Vert{\pia}(\cdot|s)-\hat{\pi}(\cdot|s)\Vert_1]+2{\lstatt}+4\epsilon_{UO}\\
\leq&2{\lstatt}+10\epsilon_{UO},\label{eq:thm3-6}
\end{align} 
where the first step comes from (\ref{eq:thm3-2}), the second is due to (\ref{eq:thm3-4}), the third and fourth is from the definition of $\bar{h}'$ and $\bar{\pi}$, the fifth step utilizes (\ref{eq:thm3-2}) and (\ref{eq:thm3-4}), the sixth step is due to $\Vert\tilde{h}\Vert_{\infty}\leq1$ and the last step is from (\ref{eq:thm3-1}) and Theorem~\ref{thm:sample regularized approximate}.

Combining (\ref{eq:thm3-5}) and (\ref{eq:thm3-6}), we have conditioned on $E$, with at least probability $1-\frac{\delta}{2}$, we have
\begin{equation}
\mathbb{E}_{s\sim{\da}}[\Vert\bar{\pi}(\cdot|s)-{\pia}(\cdot|s)\Vert_1]\leq2{\lstatt}+20\epsilon_{UO}.
\end{equation}
Notice that $\epsilon_{UO}\leq 2^{\frac{5}{4}}\sqrt{\frac{\SE_{n_1,n_0,\alpha}(\BW,\BF,\BV,\BE)}{\alpha M_f}}$. Therefore, with at least probability $1-\delta$, we have:
	\begin{align}
		&\mathbb{E}_{s\sim {\da}}[\Vert{\pia}(\cdot|s)-\bar{\pi}(\cdot|s)\Vert_1]\leq4{\BW}\sqrt{\frac{6\log\frac{4|\Pi|}{\delta}}{n_2}}+50\sqrt{\frac{\SE_{n_1,n_0,\alpha}(\BW,\BF,\BV,\BE)}{\alpha M_f}}
\end{align}	
This finishes our proof.

\subsection{Proof of Lemma~\ref{lem:importance sampling close}}
\label{proof:lem importance sampling close}
The proof is similar to Lemma~\ref{lem:tilde w tilde pi}. First notice that
\begin{align}
&\bigg|\left(\sum_{s',a'}d^D(s',a')\hat{w}(s',a')\right)-1\bigg|\\
=&\bigg|\left(\sum_{s',a'}d^D(s',a')\hat{w}(s',a')\right)-\left(\sum_{s',a'}d^D(s',a'){\wa}(s',a')\right)\bigg|\\
=&\bigg|\sum_{s',a'}d^D(s',a')\left(\hat{w}(s',a')-{\wa}(s',a')\right)\bigg|\\
\leq&\sum_{s',a'}d^D(s',a')|\hat{w}(s',a')-{\wa}(s',a')|\\
\leq&\Vert\hat{w}-{\wa}\Vert_{2,d^D}\\
\leq&\epsilon_{UO},\label{eq:thm3-3}
\end{align}
which proves the second part of the lemma. For the first part, we have
\begin{align}
&\Vert \hat{d}'-{\da}\Vert_1\\
=&\sum_s\bigg|\frac{1}{\sum_{s',a'}d^D(s',a')\hat{w}(s',a')}\sum_{a'}d^D(s,a')\hat{w}(s,a')-{\da}(s)\bigg|\\
\leq&\sum_s\left(\bigg|\frac{1}{\sum_{s',a'}d^D(s',a')\hat{w}(s',a')}-1\bigg|\sum_{a'}d^D(s,a')\hat{w}(s,a')\right)\notag\\
&+\sum_s\bigg|\sum_{a'}d^D(s,a')\hat{w}(s,a')-{\da}(s)\bigg|\\
=&\underbrace{\sum_s\left(\bigg|\frac{1}{\sum_{s',a'}d^D(s',a')\hat{w}(s',a')}-1\bigg|\sum_{a'}d^D(s,a')\hat{w}(s,a')\right)}_{(1)}\notag\\
&+\underbrace{\sum_s\bigg|\sum_{a'}d^D(s,a')\hat{w}(s,a')-\sum_{a'}d^D(s,a'){\wa}(s,a')\bigg|}_{(2)}.
\end{align}

For term (1), notice that
\begin{equation}
\bigg|\frac{1}{\sum_{s',a'}d^D(s',a')\hat{w}(s',a')}-1\bigg|=\frac{\big|1-\sum_{s',a'}d^D(s',a')\hat{w}(s',a')\big|}{\sum_{s',a'}d^D(s',a')\hat{w}(s',a')}\leq\frac{\epsilon_{UO}}{\sum_{s',a'}d^D(s',a')\hat{w}(s',a')}.
\end{equation}
Therefore,
\begin{align}
&\sum_s\left(\bigg|\frac{1}{\sum_{s',a'}d^D(s',a')\hat{w}(s',a')}-1\bigg|\sum_{a'}d^D(s,a')\hat{w}(s,a')\right)\\
\leq&\epsilon_{UO}\sum_{s}\frac{\sum_{a'}d^D(s,a')\hat{w}(s,a')}{\sum_{s',a'}d^D(s',a')\hat{w}(s',a')}\\
=&\epsilon_{UO}.
\end{align}

For term (2),
\begin{align}
&\sum_s\bigg|\sum_{a'}d^D(s,a')\hat{w}(s,a')-\sum_{a'}d^D(s,a'){\wa}(s,a')\bigg|\\
\leq&\sum_{s,a'}d^D(s,a')|\hat{w}(s,a')-{\wa}(s,a')|\\
\leq&\epsilon_{UO}.
\end{align}

Thus we have
\begin{equation}
\Vert \hat{d}'-{\da}\Vert_1\leq2\epsilon_{UO}.
\end{equation}
\section{Proof of Corollary~\ref{cor:sample unregularized 0}}
\label{proof:cor sample unregularized 0}
First by Lemma~\ref{lem:hat L tilde L}, we know that
\begin{equation}
{\Lz}({\vz},{\wz})-{\Lz}({\vz},\hat{w})\leq \frac{2{\BWZ}}{1-\gamma}\sqrt{\frac{2\log\frac{4|{\VC}||{\WC}|}{\delta}}{n}}+\sqrt{\frac{2\log\frac{4|{\VC}|}{\delta}}{n_0}}.
\end{equation}

Substitute the definition (\ref{prob:maximin2}) of ${\Lz}({\vz},w)=(1-\gamma)\mathbb{E}_{s\sim\mu_0}[{\vz}(s)]+\mathbb{E}_{(s,a)\sim d^D}[w(s,a)e_{{\vz}(s)}(s,a)]$ into the above inequality, we have
\begin{equation}
\label{eq:cor-2-eq2}
\sum_{s,a}\left({\dz}(s,a)e_{{\vz}(s)}(s,a)\right)-\sum_{s,a}\left(\hat{d}(s,a)e_{{\vz}(s)}(s,a)\right)\leq \frac{2{\BWZ}}{1-\gamma}\sqrt{\frac{2\log\frac{4|{\VC}||{\WC}|}{\delta}}{n}}+\sqrt{\frac{2\log\frac{4|{\VC}|}{\delta}}{n_0}}.
\end{equation}

Note that ${\vz}$ is the optimal value function of the unregularized MDP $\mathcal{M}$ and ${\dz}$ is the discounted state visitation distribution of the optimal policy ${\piz}$ \citep{puterman1994markov}. Therefore, invoking Lemma~\ref{lem:performance difference}, we have
\begin{align}
J(\pi)-J({\piz})&=\mathbb{E}_{(s,a)\sim d^{\pi}}[r(s,a)+\gamma\mathbb{E}_{s'\sim P(\cdot|s,a)}{\vz}(s')- {\vz}(s)]\notag\\
&=\sum_{s,a}d^{\pi}(s,a)e_{{\vz}(s)}(s,a)\label{eq:cor-2-eq1}.
\end{align}

Let $\pi=\tilde{\pi}^*_{0}$ in (\ref{eq:cor-2-eq1}), then we can obtain
\begin{equation}
\sum_{s,a}{\dz}(s,a)e_{{\vz}(s)}(s,a)=0.
\end{equation}
Substitute it into (\ref{eq:cor-2-eq2}),
\begin{equation}
\label{eq:cor-2-eq3}
\sum_{s,a}\left(\hat{d}(s,a)(-e_{{\vz}(s)}(s,a))\right)\leq \frac{2{\BWZ}}{1-\gamma}\sqrt{\frac{2\log\frac{4|{\VC}||{\WC}|}{\delta}}{n}}+\sqrt{\frac{2\log\frac{4|{\VC}|}{\delta}}{n_0}}.
\end{equation}

Notice that since ${\vz}$ is the optimal value function, $-e_{{\vz}(s)}(s,a)\geq0$ for all $s,a$. Therefore, we have:
\begin{align}
J({\piz})-J(\hat{\pi})&=\sum_{s,a}d^{\hat{\pi}}(s,a)(-e_{{\vz}(s)}(s,a))\\
&=\sum_{s,a}d^{\hat{\pi}}(s)\hat{\pi}(a|s)(-e_{{\vz}(s)}(s,a))\\
&\leq{{\BS}}\sum_{s,a}d^D(s)\hat{\pi}(a|s)(-e_{{\vz}(s)}(s,a))\\
&={{\BS}}\sum_{s,a}d^D(s)\frac{\hat{w}(s,a)\pi_D(a|s)}{\sum_{a'}\hat{w}(s,a')\pi_D(a'|s)}(-e_{{\vz}(s)}(s,a))\\
&\leq\frac{{\BS}}{{\BST}}\sum_{s,a}d^D(s)\pi_D(a|s)\hat{w}(s,a)(-e_{{\vz}(s)}(s,a))\\
&=\frac{{\BS}}{{\BST}}\sum_{s,a}\hat{d}(s,a)(-e_{{\vz}(s)}(s,a))\\
&\leq\frac{2{\BWZ}{\BS}}{(1-\gamma){\BST}}\sqrt{\frac{2\log\frac{4|{\VC}||{\WC}|}{\delta}}{n}}+\frac{{\BS}}{{\BST}}\sqrt{\frac{2\log\frac{4|{\VC}|}{\delta}}{n_0}},
\end{align}
where the first step comes from (\ref{eq:cor-2-eq1}), the third step is due to Assumption~\ref{ass:strong conc}, the fifth step comes from Assumption~\ref{ass:W good} and the last step comes from (\ref{eq:cor-2-eq3}). This concludes our proof.

\subsection{Proof of Lemma~\ref{lem:ergodic}}
\label{proof:lem ergodic}
First notice that $d^D(s)\geq(1-\gamma)\mu_0(s)$. Then since $d^{\pi}(s)\leq B_{erg,2}\mu_0(s),\forall s,\pi$, we have for any policy $\pi$:
\begin{equation}
\frac{d^{\pi}(s)}{d^D(s)}\leq\frac{1}{1-\gamma}\frac{d^{\pi}(s)}{\mu_0(s)}\leq\frac{B_{erg,2}}{1-\gamma}.
\end{equation}
On the other hand, ${\dz}(s)\geq(1-\gamma)\mu_0(s)$, therefore similarly we have:
\begin{equation}
\frac{{\dz}(s)}{d^D(s)}\geq\frac{(1-\gamma)\mu_0(s)}{d^D(s)}\geq\frac{1-\gamma}{B_{erg,2}}.
\end{equation}
\end{document}